\documentclass[10pt,twocolumn,letterpaper]{article}

\usepackage{cvpr}              % To produce the CAMERA-READY version
\usepackage[utf8]{inputenc} % allow utf-8 input
\usepackage[T1]{fontenc}    % use 8-bit T1 fonts
\usepackage{url}            % simple URL typesetting
\usepackage{booktabs}       % professional-quality tables
\usepackage{amsfonts}       % blackboard math symbols
\usepackage{nicefrac}       % compact symbols for 1/2, etc.
\usepackage{microtype}      % microtypography
\usepackage{rotating}
\usepackage{diagbox}
\usepackage{epsfig}
\usepackage{graphicx}
\usepackage{amsmath}
\usepackage{amssymb}
\usepackage{amsthm}

\usepackage{booktabs}
\usepackage{enumerate}
\usepackage{enumitem}
\usepackage[longtable]{multirow}
\usepackage{longtable}
\usepackage{booktabs}

\usepackage{color}
\usepackage{algorithm}

\usepackage{adjustbox}
\usepackage{multirow}
\usepackage{float}
\usepackage{subcaption}
\usepackage[noend]{algpseudocode}
\usepackage{wrapfig}
\usepackage{setspace}
\usepackage{xcolor}
\usepackage{colortbl}
\usepackage{makecell}
\usepackage{graphicx}
\usepackage{lipsum}

\newcommand{\B}[1]{\textbf{#1}}
\newcommand{\I}[1]{\textit{#1}}

\definecolor{Gray}{gray}{0.90}
\definecolor{LightCyan}{rgb}{0.88,1,1}

\DeclareMathOperator{\softmax}{softmax}

\definecolor{cvprblue}{rgb}{0.21,0.49,0.74}
\usepackage[pagebackref,breaklinks,colorlinks,allcolors=cvprblue]{hyperref}

\def\paperID{9458} % *** Enter the Paper ID here
\def\confName{CVPR}
\def\confYear{2025}

\author{Li Ren, Chen Chen, Liqiang Wang, Kien Hua\\
Department of Computer Science\\
University of Central Florida, USA\\
{\tt\small \{Li.Ren, Chen.Chen, Liqiang.Wang, Kien.Hua\}@ucf.edu}
}

\title{DA-VPT: Semantic-Guided Visual Prompt Tuning for Vision Transformers}
\begin{document}
\maketitle

\begin{abstract}

Visual Prompt Tuning (VPT) has become a promising solution for Parameter-Efficient Fine-Tuning (PEFT) approach for Vision Transformer (ViT) models by partially fine-tuning learnable tokens while keeping most model parameters frozen. Recent research has explored modifying the connection structures of the prompts. However, the fundamental correlation and distribution between the prompts and image tokens remain unexplored. In this paper, we leverage \textit{metric learning} techniques to investigate how the distribution of prompts affects fine-tuning performance. Specifically, we propose a novel framework, \textbf{D}istribution \textbf{A}ware \textbf{V}isual \textbf{P}rompt Tuning (DA-VPT), to guide the distributions of the prompts by learning the distance metric from their class-related semantic data. Our method demonstrates that the prompts can serve as an effective bridge to share semantic information between image patches and the class token. We extensively evaluated our approach on popular benchmarks in both recognition and segmentation tasks. The results demonstrate that our approach enables more effective and efficient fine-tuning of ViT models by leveraging semantic information to guide the learning of the prompts, leading to improved performance on various downstream vision tasks. The code is released on \url{https://github.com/Noahsark/DA-VPT}.

\end{abstract}

%%%% ============================== Introduction =============================================

\begin{figure}[ht]
\centering
	\includegraphics[width = 0.45\textwidth]{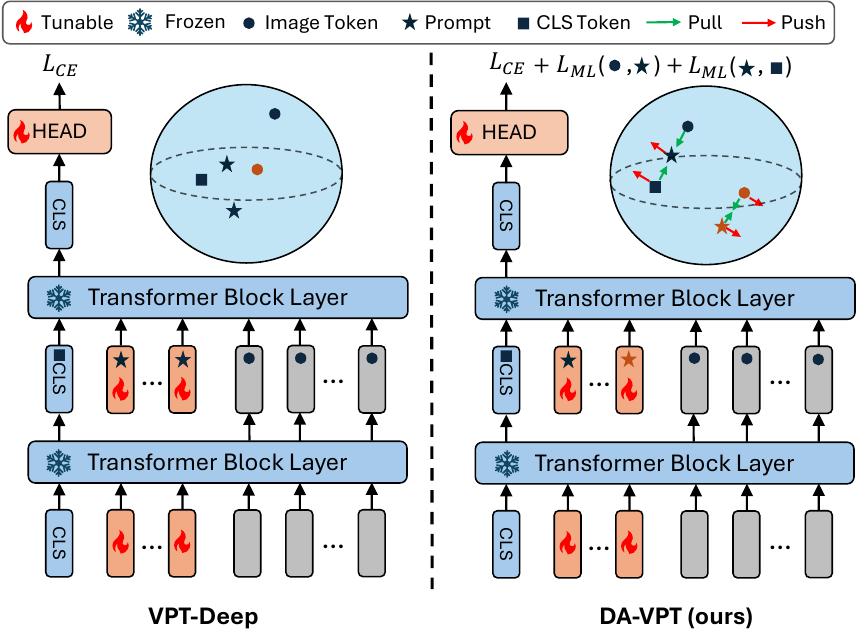}
\caption{
\small
\textbf{Comparison between VPT-Deep and DA-VPT.} \textbf{Left (VPT-Deep):} Prompts are guided solely by the recognition task, leading to unconstrained distributions between prompts and visual tokens. This allows prompts to attract features from arbitrary classes, potentially hindering the class token's ability to aggregate class-specific information. \textbf{Right (DA-VPT):} Prompts are jointly optimized by the main task and semantic metric learning objectives. The semantic clustering aligns the distributions of prompts, visual tokens, and class tokens, enabling more effective class-specific information aggregation through semantically-guided attention.
}

\label{fig:introduction}
\end{figure}

\section{Introduction} \label{sec:intro}

Recent advances in model scaling and dataset expansion \citep{deng2009imagenet, sun2017revisiting, mahajan2018exploring} have led to powerful vision foundation models, particularly those based on Vision Transformer (ViT) architectures \citep{vit_2020}. These models have demonstrated exceptional performance across various computer vision tasks \citep{he2020momentum, radford2021learning, he2022masked}. While fine-tuning these models for downstream tasks like visual recognition \citep{vit_2020} or semantic segmentation \citep{kirillov2023segment} has become standard practice, the conventional full fine-tuning approach faces significant challenges, including high computational costs, overfitting, and catastrophic forgetting \citep{kornblith2019better, nguyen2019toward}. These challenges have motivated the development of Parameter-Efficient Fine-Tuning (PEFT) methods that selectively update a small subset of model parameters while keeping the majority frozen \citep{kornblith2019better, adapter19, 2020adapterhub, zaken2021bitfit, li2021prefix, adaptformer22, jia2022visual, ren2024learning}.

Initially emerging from the NLP domain, \citet{adapter19} and subsequent works \citep{2020adapterhub, hu2021lora} demonstrated that updating a minimal number of parameters could achieve performance comparable to full fine-tuning. These techniques were later adapted to computer vision by \citet{adaptformer22}, who introduced parallel residual networks alongside the ViT backbone. A significant advancement came from \citet{jia2022visual}, who proposed Visual Prompt Tuning (VPT). This method introduces learnable tokens called \textit{visual prompts} at the input level of each ViT layer, effectively aligning downstream task distributions with pre-training data distributions through learnable data-level representations. Building on this foundation, \citet{yoo2023improving} and \citet{han20232vpt} further enhanced VPT by implementing cross-layer prompt connections with dynamic gating mechanisms, enabling adaptive control of prompt positioning and quantity.

However, existing VPT approaches primarily focus on manipulating prompt connections and structure, while overlooking the intrinsic relationship between prompts and data representations. Specifically, current VPT and its related methods \citep{jia2022visual, yoo2023improving, han20232vpt, pei2024sa2vp} initialize prompts randomly and optimize them solely through downstream task objectives. Although recent work \citep{wang2024revisiting} demonstrates improved learning efficiency through data-driven prompt initialization, the potential of leveraging discriminative and class-aware information remains largely unexplored. To deepen our understanding of prompt functionality and distribution, we investigate prompt-token relationships by addressing a fundamental question: \textbf{Could prompts be guided to facilitate information flow between image and class tokens to enhance representation learning?}

To address this question, we introduce a novel method that guides VPT optimization by leveraging semantic connections between visual prompts, visual tokens, and class tokens. We propose connecting prompts and visual data by constructing and learning a semantic metric between them in the deep layers of the ViT. For each prompt in these layers, we establish a semantic connection with its closest labeled class. As illustrated in Figure~\ref{fig:introduction}, we construct a semantic metric in the latent space by comparing prompts with corresponding image patches. Specifically, we aim to minimize the distance between visual prompts and visual tokens of the same class while maximizing separation from the visual tokens of different classes. We also apply a similar semantic metric between class tokens and prompts.

Our key insight is to increase the likelihood that prompts capture semantic information from same-class visual tokens while filtering out unrelated information. Through semantic metrics in both visual feature and prompt spaces, we demonstrate the effective transfer of relevant semantic information from visual tokens to class tokens via class-specific prompts. In other words, our framework employs related prompts as a \textit{bridge} to connect class tokens and image patch semantic information through guided attention maps.

Extensive experiments across 24 visual recognition tasks in both Fine-Grained Visual Classification (FGVC) \citep{jia2022visual} and Visual Task Adaptation Benchmark (VTAB-1k) \citep{zhai2019large} demonstrate substantial improvements over standard VPT. Our method shows consistent effectiveness with both supervised and self-supervised pre-trained models, including MoCo and MAE. Additional evaluations on segmentation tasks further confirm that our approach significantly improves prompt learning efficiency and downstream task performance while requiring fewer prompts and learnable parameters compared to baseline VPT and its related state-of-the-art methods.
Our main contributions are:

\begin{itemize}[leftmargin=*]
\item We propose \textbf{D}istribution \textbf{A}ware \textbf{V}isual \textbf{P}rompt \textbf{T}uning (DA-VPT), a novel framework that enhances prompt learning by constructing semantic metrics between prompts and corresponding image feature patches in deep ViT layers.
\item We demonstrate that prompts can effectively bridge semantic information between image patches and class tokens through the attention mechanism, highlighting the importance of guided prompt learning.
\item We validate our method's effectiveness through extensive experiments on 24 visual recognition tasks and 2 segmentation tasks, showing significant improvements over vanilla VPT and its related works for both supervised and self-supervised pre-trained vision models.
\end{itemize}
%%%% =============================== Related Works ========================================

\section{Related Works}
\label{sec:related}

\B{Parameter-Efficient Fine-Tuning (PEFT)}

Transformers, initially introduced by \citet{transformer_2017}, have revolutionized various domains through pre-training, from natural language processing (e.g., LLaMA \citep{touvron2023llama}, GPT \citep{brown2020language}) to computer vision (e.g., MAE \citep{mae_2022}, CLIP \citep{clip_2021}, ViT-22b \citep{dehghani2023scaling}). PEFT approaches have emerged to address the computational challenges of fine-tuning these large models by selectively updating only a subset of parameters. Early work by \citet{kornblith2019better} focused on training only the classification head, while \citet{zaken2021bitfit} demonstrated significant improvements by tuning bias terms alone. \citet{lian2022scaling} and \citet{xie2023difffit} further refined these approaches by introducing adjustable shifting and scaling factors. Another significant direction in PEFT, pioneered by \citet{adapter19}, involves incorporating lightweight adapter modules alongside Transformer backbones.

\B{Visual Prompt Tuning (VPT)} 
As a prominent branch of PEFT, prompt tuning introduces learnable tokens alongside input data to incorporate task-specific information \citep{li2021prefix, liu2023pre, lester2021power, liu2021p}. \citet{jia2022visual} pioneered the application of prompts in Vision Transformers (ViT), introducing VPT-Shallow for input layer modification and VPT-Deep for cross-layer integration. This foundational work catalyzed numerous developments in the field: \citet{gao2022visual} adapted visual prompts for test-time domain adaptation, while \citet{compositionalvpt2023} extended the approach to video recognition. Subsequent studies enhanced VPT's capabilities through dynamic mechanisms for optimizing prompt quantity and placement \citep{han20232vpt, yoo2023improving}, direct connections between intermediate layers and task-specific heads \citep{tu2023visual}, and spatial selection mechanisms for coordinating attention between image patches and visual prompts \citep{pei2024sa2vp}.

\B{Integration of PEFT Approaches}
While recent comprehensive approaches have demonstrated success in combining multiple PEFT methods \citep{chavan2023one, zhang2022neural}, our work focuses specifically on integrating bias optimization with VPT, which we empirically found to be sufficiently effective for demonstrating our method's capabilities. A comprehensive evaluation of combinations with other PEFT methods lies beyond the scope of this work.

\B{Metric Learning (ML)} 
Metric Learning focuses on learning representations that effectively capture similarities and differences between data samples in the embedding space. Early approaches employed \I{contrastive loss} to differentiate between class samples \citep{chopra2005learning, hadsell2006}. This evolved into \I{triplet loss} methods that introduce an \I{anchor} point as a proxy to simultaneously compare positive and negative samples with specified margins \citep{cheng2016person, kim2020proxy, ren2024towards}.

Advanced metric learning techniques have incorporated \I{Neighbourhood Components Analysis (NCA)} to better understand data distributions and class relationships \citep{roweis2004nca, movshovitz2017no, teh2020proxynca++, kim2020proxy, mix2021, roth2022non}. Recent studies have demonstrated particular success in applying NCA-based metric learning to Vision Transformer architectures \citep{hyp2022, recallk_2022, kotovenko2023cross}, emphasizing the crucial role of data distributions in learning discriminative representations \citep{wang2017adversarial, laradji2020m, ren2021beyond, ren2024towards}. Further investigations by \citet{tsai2024convolutional} and \citet{ren2024learning} have explored the integration of visual prompts with robust visual perception and deep metric learning.

Building on metric learning, our work examines the interactions between visual prompts, visual tokens, and class tokens within ViTs. We bridge the gap between traditional metric learning techniques and modern visual prompt tuning methods, offering a more principled way to optimize prompt-based transfer learning.

%%%% =============================== Preliminary & Methodology =========================================

\section{Methodology} \label{sec:method}

\subsection{Preliminary} \label{sec:preliminary}

\textbf{The Vision Transformer (ViT)} \citep{vit_2020} is a fundamental model architecture that applies the original Transformer model \citep{vaswani2017attention} to computer vision tasks. Given an input image $\mathbf{I} \in \mathbb{R}^{H \times W \times C}$, ViT divides it into a sequence of $N$ flattened 2D patches, which are then linearly projected into a $D$-dimensional embedding space. A learnable [CLS] (Class) token $\mathbf{x}_\text{cls} \in \mathbb{R}^D$ is prepended to the patch embeddings, serving as a global representation for classification tasks. The resulting sequence of embeddings $\mathbf{X} \in \mathbb{R}^{(N+1) \times D}$ is then passed through $L$ Transformer block layers, where $l \in \{1,\ldots,L\}$ denotes the layer index. Each layer consists of a Multi-Head Self-Attention (MHSA) mechanism defined as $\text{MHSA}(\mathbf{X}^{l}) = \text{Concat}(\mathbf{H}_1, \cdots, \mathbf{H}_h)$, where each head $\mathbf{H}_i$ computes a scaled dot-product attention $\softmax(\frac{\mathbf{Q}\mathbf{K}^T}{\sqrt{d}}\mathbf{V})$ with subspaces of Query ($\mathbf{Q}$), Key ($\mathbf{K}$), and Value ($\mathbf{V}$) matrices projected from input embedding $\mathbf{X}^{l-1}$ in the previous layer. The final output is the [CLS] token $\mathbf{x}_\text{cls}^L$, used for downstream classification tasks.

\textbf{Visual Prompt Tuning (VPT)} \citep{jia2022visual} presents a promising PEFT technique for ViT that adapts the pre-trained model to downstream tasks by introducing a small set of learnable parameters, namely \textit{visual prompts}. In a specific ViT block layer, a sequence of $M$ learnable prompt tokens $\mathbf{P} = \{\mathbf{p}_1, \dots, \mathbf{p}_M\} \in \mathbb{R}^{M \times D}$ is concatenated with the patch embeddings $\mathbf{X} = \{\mathbf{x}_1, \dots, \mathbf{x}_N\} \in \mathbb{R}^{N \times D}$. \citet{jia2022visual} propose two VPT settings: \textbf{VPT-Shallow} where the prompts are only inserted into the first ViT layer, and \textbf{VPT-Deep} where the prompts are appended into every ViT layer. We follow the \textbf{VPT-Deep} setting since it has a higher capacity and aligns with our proposed method. The resulting sequence of embeddings $[\mathbf{x}_\text{cls}, \mathbf{P}, \mathbf{X}] \in \mathbb{R}^{(M+N+1) \times D}$ is then processed by the next ViT encoder layers. For layer $l$, the output of the $(l+1)$-th layer can be described as:

\begin{equation}
\setlength{\abovedisplayskip}{4pt} 
\setlength{\belowdisplayskip}{4pt} 
\scalebox{0.90}
    { 
$[\mathbf{x}^{l+1}_\text{cls}, [\quad], \mathbf{x}^{l+1}_1 \dots \mathbf{x}^{l+1}_N] = \mathit{BLK}([\mathbf{x}^l_\text{cls}, \mathbf{p}^l_1 \dots \mathbf{p}^l_M, \mathbf{x}^l_1 \dots \mathbf{x}^l_N]),$
   }
\end{equation}

where $\mathbf{p}^l_1 \dots \mathbf{p}^l_M$ are the $M$ prompts in layer $l$, $\mathit{BLK}$ represents the transformer block, and $[\quad]$ represents the position reserved for prompts in the next layer. During fine-tuning, only the visual prompts $\mathbf{P}$ and the linear classification head are updated.

\textbf{Metric Learning (ML)} aims to learn a distance metric that captures semantic similarity between data points. The \textit{Neighborhood Component Analysis} (NCA) \citep{roweis2004nca} encourages learned embeddings to have a higher probability of correct classification by nearest neighbor classifiers. Given a set of $N$ labeled data points $\{(\mathbf{x}_i, y_i)\}_{i=1}^N$, where $\mathbf{x}_i \in \mathbb{R}^D$ is the input feature vector and $y_i \in \{1,\ldots,C\}$ is the corresponding class label from $C$ classes, the NCA objective is:

\begin{equation} 
\setlength{\abovedisplayskip}{4pt} 
\setlength{\belowdisplayskip}{4pt} 
\mathcal{L}_\text{NCA} = -\sum_{i=1}^N \log \frac{\sum_{j \in \mathcal{N}_i} \exp(-D(\mathbf{x}_i, \mathbf{x}_j)/\tau)}{\sum_{k \neq i} \exp(-D (\mathbf{x}_i, \mathbf{x}_k)/\tau)}, 
\end{equation} 

where $\mathcal{N}_i = \{ j \mid y_j = y_i, j \neq i \}$ denotes the set of neighboring points with the same class, $\tau > 0$ is the temperature parameter, and $D(\cdot, \cdot)$ represents the \textit{cosine similarity}: $D(\mathbf{x}_i, \mathbf{x}_j) = \hat{\mathbf{x}}_i \cdot \hat{\mathbf{x}}_j$ where $\hat{\mathbf{x}}=\frac{\mathbf{x}}{\|\mathbf{x}\|_2}$ represents the L2-normalized vector. Following NCA, recent works \citep{teh2020proxynca++, kim2020proxy} introduce learnable class representations $\mathbf{P}=\{\mathbf{p}_i \in \mathbb{R}^D\}_{i=1}^C$, named \textit{proxies}, to represent the $C$ classes. In our work, we propose using prompts in deep layers as proxies for subsets of semantically similar classes.

\begin{figure}[t]
\centering
	{\includegraphics[width = 0.50\textwidth]{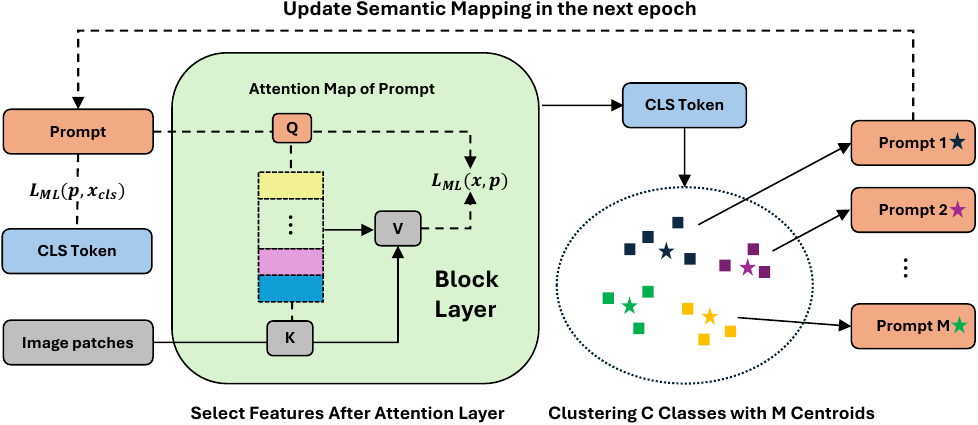}} \hfill
\caption{
\small
\textbf{Framework Overview.} Our method establishes semantic prompt-class mappings by clustering class representations into $M$ clusters ($M$ = number of prompts). Prompts are guided through a metric space using smoothed proxy NCA loss $\mathcal{L}_\text{ML}$ between prompts and attention-based output tokens, enabling each prompt to capture information from its assigned semantic cluster. A similar metric guides [CLS] token-prompt relationships. The semantic mapping updates after each epoch, optimizing prompt distribution to capture fine-grained class-specific features.}

\label{fig:framework}
\end{figure}

%%%% =============================== Methology =================================================================================================================

\subsection{Metric Learning on the Learnable Prompts} \label{sec:method_ml} 

Our objective is to establish a metric in the feature space that quantifies the distance between learnable prompts and either visual tokens or the [CLS] token. We hypothesize that within each layer, a specific prompt should selectively capture information from a subset of relevant classes rather than searching indiscriminately across the entire class space. This targeted approach enables prompts to become more discriminative in their feature extraction while optimizing the [CLS] token's ability to aggregate task-specific information from each class effectively. While this structured information capture may differ from the emergent behavior of standard visual prompts, our empirical results demonstrate that metric learning guidance enhances model transferability, particularly when applied to deeper layers.

For a ViT block $\mathit{BLK}_l$ at layer $l$ ($l > 0$), we regularize the learning of prompts $\mathbf{P}^l$ by constructing a space metric between the normalized prompts $\hat{\mathbf{p}}_k^l$ and normalized visual tokens $\hat{\mathbf{x}}_i^l$. For each prompt $\hat{\mathbf{p}}_k^l \in \mathbf{P}^l$ with assigned class label $y_k$, we aim to satisfy the following constraint for visual token samples $\hat{\mathbf{x}}_i^l$ and $\hat{\mathbf{x}}_j^l$ in the same batch with class different labels $y_i$ and $y_j$ respectively, where $\hat{\mathbf{x}}_i^l$ shares the same class label $y_i=y_k$ with $\hat{\mathbf{p}}_k^l$:

\begin{equation}
\setlength{\abovedisplayskip}{4pt} 
\setlength{\belowdisplayskip}{4pt} 
\hat{\mathbf{p}}_k^l \cdot \hat{\mathbf{x}}_i^l - \delta \geq \hat{\mathbf{p}}_k^l \cdot \hat{\mathbf{x}}_j^l + \delta \quad \forall i, j, k, y_k=y_i \neq y_j,
\label{eq:constraint_normalized}
\end{equation}

where $\cdot$ denotes the dot product and $\delta > 0$ is the pre-defined margin. This constraint ensures that the cosine similarity between a prompt and tokens of the same class is greater than the similarity with tokens of different classes. Since cosine similarity naturally aligns with attention map comparison between Query and Key vectors, we argue that pairs $(\mathbf{p}_k, \mathbf{x}_i)$ that are closer in the spherical space will have higher probability of matching in the optimized attention map.

To efficiently build a metric space satisfying this constraint, we adopt the ML loss from \citet{kim2020proxy}, comparing learnable prompts with visual tokens using smoothed NCA loss (Proxy-Anchor loss). Our metric guidance objective between visual tokens $\mathbf{X}^l$ and prompts $\mathbf{P}^l$ is:

\begin{equation}
\label{eq:loss_pa}
    \scalebox{0.90}
    {
    $\begin{aligned}
    \mathcal{L}_\text{ML}(\mathbf{X}, \mathbf{P}) = & \frac{1}{|\mathcal{P}^{+}|} \sum_{\mathbf{p}_k \in \mathcal{P}^{+}}\left[\underset{\mathbf{x}_i \in \mathcal{X}_p^{+}}{\operatorname{LSE}_0^{+}}\left(-\left(\hat{\mathbf{p}}_k \cdot \hat{\mathbf{x}}_i - \delta\right)/\tau\right)\right] + \\
    & \frac{1}{|\mathcal{P}|} \sum_{\mathbf{p}_k \in \mathcal{P}}\left[\underset{\mathbf{x}_j \in \mathcal{X}_p^{-}}{\operatorname{LSE}_0^{+}} \left(\left(\hat{\mathbf{p}}_k \cdot \hat{\mathbf{x}}_j + \delta\right)/\tau\right)\right],
    \end{aligned}$
    }
\end{equation}

where $\operatorname{LSE}_0^{+}(x) = \log(1+ \sum_{i=1}^N e^{x_i})$ is the smoothed LogSumExp with first argument set to $1$, $\mathcal{P}$ denotes the set of all prompts, $\mathcal{P}^{+}$ denotes the set of positive prompts where same-class data exists in the minibatch, $\mathcal{X}_{p}^{+}$ denotes the set of visual tokens with the same label as the selected prompt $\mathbf{p}$, and $\mathcal{X}_{p}^{-}$ is its complement set. In practice, we found that comparing the projected Query vector $\mathbf{Q} = \mathbf{P}^l\mathbf{W}_{Q}^l$ yields better performance, where $\mathbf{W}_{Q}^l \in \mathbb{R}^{D \times D}$ is the Query projection matrix at layer $l$.

We also propose a similar loss $\mathcal{L}_\text{ML}(\mathbf{P}, \mathbf{x}_\text{cls})$ that pulls the [CLS] token closer to corresponding prompts while pushing it away from prompts of different classes. The overall loss becomes:
\begin{equation}
\setlength{\abovedisplayskip}{5pt} 
\setlength{\belowdisplayskip}{5pt} 
\mathcal{L} = \mathcal{L}_\text{CE} + \beta \mathcal{L}_\text{ML}(\mathbf{X}, \mathbf{P}) + \lambda \mathcal{L}_\text{ML}(\mathbf{P}, \mathbf{x}_\text{cls}),
\end{equation}
where $\beta, \lambda > 0$ are hyperparameters. By jointly optimizing both metric learning terms, our method encourages prompts to capture class-specific information and aligns the [CLS] token with relevant prompts.

\subsection{Projection and Saliency Patch Selection}
\label{sec:saliency}

To ensure prompts effectively focus on critical image information while filtering out false positive visual tokens, we propose selecting saliency information from visual tokens as positive and negative samples for prompt comparison in $\mathcal{L}_\text{ML}(\mathbf{X}, \mathbf{P})$. While extracting saliency patches directly from attention maps is straightforward, it can be computationally intensive, especially with optimized attention mechanisms like Flash Attention. Instead, as shown in Figure~\ref{fig:framework}, we use the output representation immediately following the attention layer. The output representation $\mathbf{X}^l = \text{MHSA}(\mathbf{X}^l) \in \mathbb{R}^{N \times D}$ then concatenates representations from each head, serving as a saliency aggregation of visual tokens.

\subsection{Dynamically Mapping Classes and Prompts}
\label{sec:mapping}

We set $M$ learnable prompts in each layer where $M \ll C$ to avoid optimization difficulties and unequal training opportunities. We develop a semantic mapping strategy to map $C$ classes to $M$ prompts. Before training, we run an additional epoch to obtain class representations $\mathbf{S} \in \mathbb{R}^{C \times D}$ by mean-pooling the [CLS] token for each class using the pre-trained ViT. We then use k-means clustering to group these representations into $M$ clusters, assigning classes to prompts based on cluster membership, as shown in Figure~\ref{fig:framework}.

To maintain semantic mapping accuracy, we update the mapping after each epoch. During training, we collect and calculate updated class representations $\mathbf{S}$, then update k-means using previous epoch centroids as initialization to adjust the class-prompt mapping.

\subsection{Efficient Bias Tuning} 
\label{sec:bias}

To further improve the flexibility in the distribution of visual tokens, we investigate the partial release of ViT backbone bias terms as suggested by \citet{zaken2021bitfit}. We found that fine-tuning performance significantly improves when bias terms are partially enabled with our metric guidance loss. The most efficient components are the bias terms $\mathbf{b}_K, \mathbf{b}_V \in \mathbb{R}^D$ in the Key and Value linear projections of the self-attention mechanism (Figure~\ref{fig:impact_1:efficient_bias}). This observation aligns with findings from \citet{zaken2021bitfit} and \citet{cordonnier2020multi}. Partially allowing bias terms to adapt provides additional flexibility in adjusting visual token distributions and capturing task-specific information under metric guidance.

\newtheorem{theorem}{Theorem}
\section{Technical Discussion} \label{sec:discussion}

\subsection{Connection Between Similarity and Attention}

In this section, we analyze how changes in token similarity influence attention weights through gradient updates. Specifically, we examine how a small change in the similarity between a prompt $\mathbf{p}$ and a visual token $\mathbf{x}_i$ affects the corresponding attention weight $a_i$. Let $\Delta \mathbf{p}$ represent a small perturbation that brings $\mathbf{p}$ closer to $\mathbf{x}_i$ in the embedding space. We formalize this relationship in the following theorem:

\begin{theorem}[Attention-Similarity Relationship]
For an attention weight perturbation $\Delta a_i$ computed using the softmax function, the following approximation holds:
\begin{equation}
\Delta a_i \approx a_i (1 - a_i) \Delta s_i,
\end{equation}
where $\Delta s_i$ represents the change in attention score $s_i$, given by $\Delta s_i = \frac{\Delta \mathbf{p}^\top \mathbf{x}_i}{\sqrt{d}}$, and $d$ is the dimension of the attention head.
\end{theorem}

This approximation reveals that a positive gradient change in attention weight ($\Delta a_i > 0$) occurs when:
\begin{equation}
\setlength{\abovedisplayskip}{2pt} 
\setlength{\belowdisplayskip}{2pt} 
a_i (1 - a_i) \Delta s_i = a_i (1 - a_i) \frac{\Delta \mathbf{p}^\top \mathbf{x}_i}{\sqrt{d}} > 0
\end{equation}
This condition is satisfied when $\mathbf{p}$ moves closer to $\mathbf{x}_i$ in the embedding space. Conversely, when $\mathbf{p}$ moves away from $\mathbf{x}_i$, $\Delta a_i$ decreases. This theorem establishes a direct connection between token similarity and attention mapping, demonstrating how our metric learning guidance influences attention through token distribution. The complete proof is provided in the Appendix.

\subsection{Analysis of Guided Attention Maps} 
\label{sec:attention}
\begin{figure*}[htpb]
\begin{subfigure}[b]{0.65\textwidth}
    \centering
	{\includegraphics[width=\textwidth]{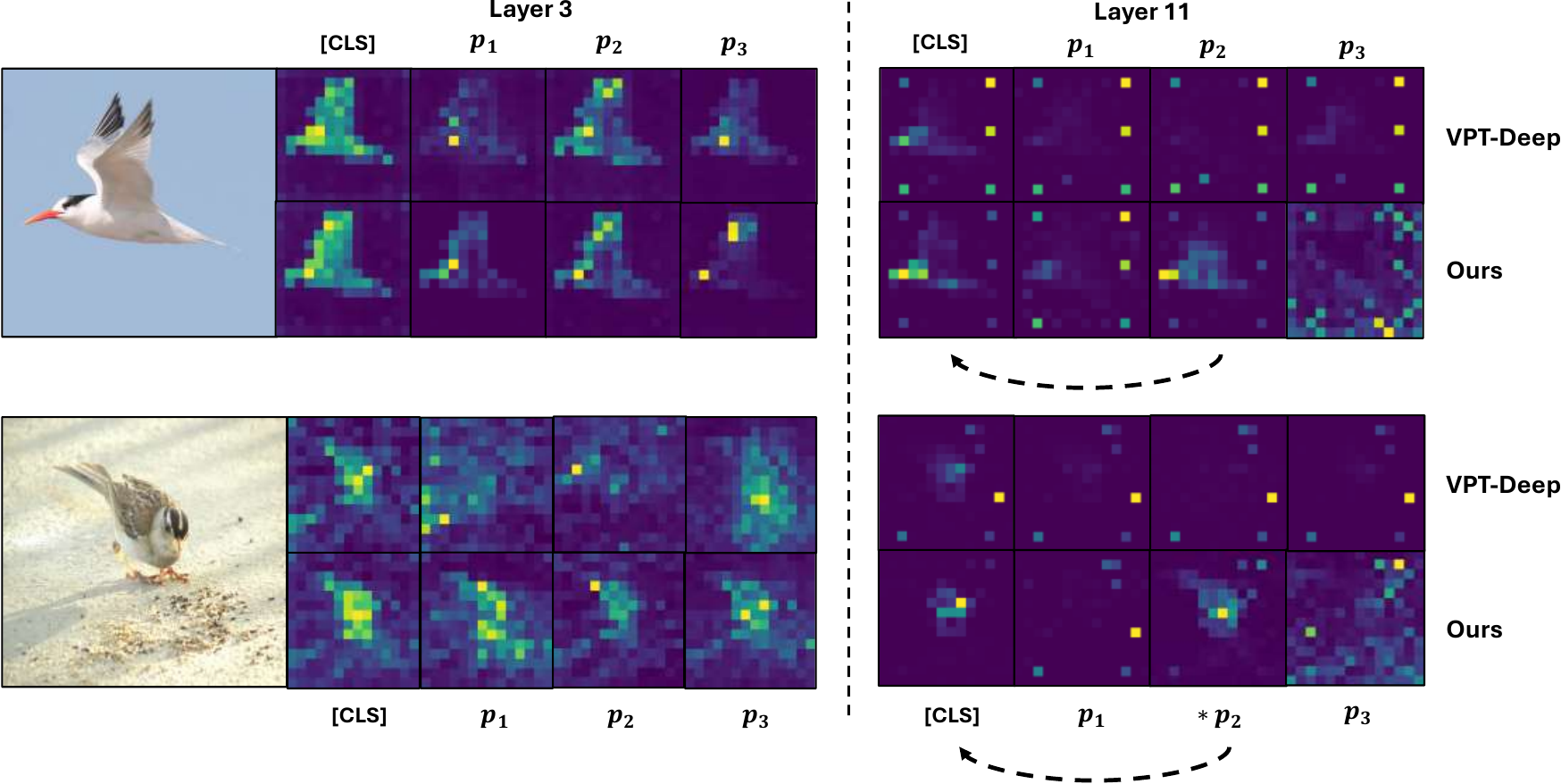}} 
    \caption{}
    \label{fig:attn_map:map}
\end{subfigure}
\hfill
\begin{subfigure}[b]{0.30\textwidth}
    \centering
	{\includegraphics[width=\textwidth]{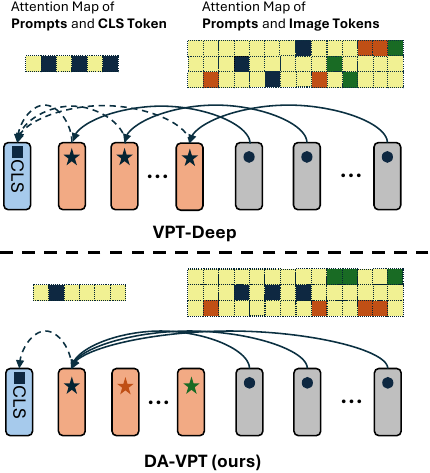}} 
    \caption{}
    \label{fig:attn_map:intro}
\end{subfigure}

\vspace{-10pt}
\caption{
\small
Comparison of attention patterns between VPT-Deep and our method on CUB dataset samples. \textbf{(a)} Attention maps in shallow (layer 3) and deep (layer 11) layers, showing CLS token and sampled prompt attention patterns. In layer 11, $\ast p$ indicates the prompt guided as positive to the CLS token, while others represent negative prompts. Additional visualization examples are provided in Appendix. \textbf{(b)} Information flow comparison between baseline VPT-Deep and our DA-VPT. Our method enables selected prompts to aggregate and transfer fine-grained details from same-class visual tokens to the [CLS] token more effectively.
}
\vspace{-5pt}
\label{fig:attn_map}
\end{figure*}

To further analyze the impact of our metric guidance loss on fine-tuning, we visualize attention maps between visual tokens and prompts across different layers (Figure~\ref{fig:attn_map:map}). Our analysis reveals distinct patterns:

\textbf{(1) In Shallow Layers}:
\begin{itemize}
    \item Both VPT-Deep and our method show prompts attending to different object subregions
    \item Our method demonstrates enhanced diversity and precision in information capture
\end{itemize}

\textbf{(2) In Deep Layers}:
\begin{itemize}
    \item Attention maps become sparser as token representations become more abstract
    \item Standard VPT-Deep prompts show limited information selection compared to the [CLS] token
    \item Our positive prompt ($\ast \mathbf{p}$) successfully identifies informative patches that are subsequently selected by the [CLS] token
\end{itemize}

These visualizations demonstrate that positively labeled prompts serve as effective "bridges" for semantic information flow to the [CLS] token in deep layers. As shown in Figure~\ref{fig:attn_map:intro}, our DA-VPT enables prompts to aggregate discriminative features from same-class data, resulting in more fine-grained attention patterns compared to the baseline. This enhanced information routing significantly improves the model's discriminative capability during fine-tuning.

\textbf{Artifact Consideration:} Recent work by \citet{darcet2023vision} revealed the existence of attention artifacts in vision transformers, which we also observe in our prompt attention maps (Figure~\ref{fig:attn_map:map}). While they demonstrate that training with \textit{registers} (analogous to learnable prompts) from scratch can eliminate these artifacts, both standard VPT and our method exhibit them due to prompt introduction during fine-tuning rather than pre-training. Nevertheless, our comparative analysis shows that the proposed guidance mechanism reduces both the frequency and spread of artifacts, constraining them more effectively within semantic object boundaries. While our method achieves better alignment between attention patterns and object semantics, the nature and impact of these artifacts on model performance presents an intriguing avenue for future investigation.

\subsection{Compatibility with Metric Learning Methods}

Our selection of the \textit{Proxy-Anchor} method \citep{kim2020proxy} is motivated by its natural alignment with our hypothesized role of visual prompts, where both proxies and prompts serve as class representatives and comparative anchors for data tokens. Alternative metric learning approaches, such as \textit{Proxy-NCA} \citep{teh2020proxynca++} and conventional \textit{triplet loss} \citep{cheng2016person, hermans2017defense}, treat all representations as equal data points. These approaches are less suitable for our framework because the significant disparity between the number of visual prompts ($M$) and data tokens ($N$, where $M \ll N$) creates an \textit{unbalanced optimization problem}. Our empirical studies further confirm this theoretical intuition: attempting to fine-tune with conventional metric learning methods leads to training instability, whereas the Proxy-Anchor formulation maintains stable optimization by explicitly accounting for the asymmetric nature of prompt-token relationships.

%%%% ================================ Experiments =============================================
\section{Experiments}
\label{sec:exp}
\subsection{Experimental Setup}
\label{sec:setup}

\textbf{Datasets.} We evaluate our method on three types of visual transfer learning tasks. For visual recognition, we use the Fine-Grained Visual Classification (FGVC) benchmark \citep{jia2022visual} comprising 5 datasets. For few-shot transfer learning, we employ the Visual Task Adaptation Benchmark (VTAB-1K) \citep{zhai2019large} containing 19 datasets. Additionally, we evaluate dense prediction tasks on ADE20K \citep{zhou2019semantic} and PASCAL Context \citep{mottaghi2014role}. Detailed dataset characteristics and experimental settings are provided in the Appendix.

\textbf{Model Architecture.} We employ Vision Transformer (ViT) \citep{vit_2020} as our backbone, using the base model \textbf{ViT-B} (12 layers) for visual classification tasks and the large model \textbf{ViT-L} (24 layers) for semantic segmentation. To evaluate generalization, we initialize the backbone using either supervised pre-training on ImageNet-21K \citep{deng2009imagenet} or self-supervised pre-training on ImageNet-1K using methods such as MoCo v3 \citep{chen2021empirical} and MAE \citep{he2022masked}.

\textbf{Method Variants.} We evaluate two versions of our approach. Our primary method, \textbf{DA-VPT}, builds on VPT-Deep \citep{jia2022visual} while incorporating our proposed metric learning losses $\mathcal{L}_\text{ML}(\mathbf{X}, \mathbf{P})$ and $\mathcal{L}_\text{ML}(\mathbf{P}, \mathbf{x}_\text{cls})$. The enhanced version, \textbf{DA-VPT+}, further incorporates efficient bias tuning as detailed in Section~\ref{sec:bias}.

\textbf{Implementation Details.} For all experiments, we conduct extensive hyperparameter optimization, including learning rate, parameter decays, and the number of visual prompts for layers both with and without our metric learning guidance. Through this empirical investigation, we identified that the optimal number of prompts for most downstream tasks is approximately 20. For metric learning parameters, we adopt the Proxy-Anchor defaults with margin $\delta = 32$ and temperature $\tau = 10$. Extended experimental details, including hyperparameter studies and ablation analyses, are provided in the Appendix.

%%%% =========================================================================================================================================================

\subsection{Result Comparison with the State-of-the-Art}
\label{sec:result}

% tab:compare_all
\begin{table}[ht]
\centering
\small
\resizebox{0.99\linewidth}{!}{
\setlength\tabcolsep{2.0pt}
    \footnotesize
    \begin{tabular}{|l|c|c|cccc|}
        \hline
        ~ & Mean & FGVC & \multicolumn{4}{c|}{VTAB-1K}  \\
        Methods & Param (M) &Mean  Acc (5) & Natural (7) & Specialized (4) & Structured (8) & Mean Acc \\
        \hline
        \multicolumn{7}{|c|}{ViT-B with Supervised pretrained on ImageNet-21k} \\
        \hline
        Full              				& 85.98 & 88.54 & 75.88 & 83.36 & 47.64 & 68.96 \\
        VPT-Shallow       		        &  0.11 & 84.62 & 76.81 & 79.68 & 46.98 & 67.82 \\
        VPT-Deep          		        &  0.64 & 89.11 & 78.48 & 82.43 & 54.98 & 71.96 \\
        E2VPT \citep{han20232vpt}  		& 0.33 & 89.22 & 80.01 & 84.43 & 57.39 & 73.94 \\
        DA-VPT (ours) 					& 0.21 & 91.22 & 80.25 & 85.12 & 58.71 & 74.69 \\
        DA-VPT+ (ours)   				&  0.24 & \textbf{91.94} & \textbf{81.98} & \textbf{86.47} & \textbf{59.96} & \textbf{76.14} \\
        \hline 
        \multicolumn{7}{|c|}{ViT-B with MAE pretrained on ImageNet-1K} \\
        \hline
        Full              									& 85.8 & 82.80 & 59.31 & 79.68 & 53.82 & 64.27\\
        VPT-Shallow       									& 0.10 & 57.84 & 39.96 & 69.65 & 27.50 & 45.70 \\
        VPT-Deep          									& 0.20 & 72.02 & 36.02 & 60.61 & 26.57 & 41.73 \\
        GateVPT \citep{yoo2023improving}          	 		& 0.17 & 73.39 & 47.61 & 76.86 & 36.80 & 53.09 \\
        E2VPT \citep{han20232vpt}  							& 0.06 & -- & 59.52 & 77.80 & 44.65 & 60.66 \\
        DA-VPT (ours) 										& 0.20 & 82.17 & 62.14 & 79.14 & 54.31 & 65.19 \\
        DA-VPT+ (ours)   									& 0.22 & \textbf{83.20}  & \textbf{66.59} & \textbf{82.96} & \textbf{59.28}  & \textbf{69.61} \\
        \hline
        \multicolumn{7}{|c|}{ViT-B with MoCo-V3 pretrained on ImageNet-1K} \\
        \hline
        Full               				         & 85.8 & 84.25 & 71.95 & 84.72 & 51.98 & 69.55 \\
        VPT-Shallow        	                    & 0.11 & 79.26 & 67.34 & 82.26 & 37.55 & 62.38 \\
        VPT-Deep           		                   & 0.20 & 83.12 & 70.27 & 83.04 & 42.38 & 65.90 \\
        GateVPT \citep{yoo2023improving}            & 0.17 & 83.00 & 74.84 & 83.38 & 49.10 & 69.11\\
        E2VPT \citep{han20232vpt}  				  & 0.11 & -- & 76.47 & \textbf{87.28} & 54.91 & 72.88 \\
        DA-VPT (ours) 							& 0.21 & 85.02 & 74.24 & 83.21 & 55.23 & 70.90 \\
        DA-VPT+ (ours)   						& 0.24 & \textbf{86.16} & \textbf{76.86} & 84.71 & \textbf{58.98} & \textbf{73.53} \\
        \hline	
    \end{tabular}
    }
\caption{
\small
\textbf{Comparison of Fine-tuning Methods.} Performance evaluation across 24 vision tasks (5 FGVC and 19 VTAB-1K) using supervised ViT and self-supervised backbones (MAE~\citep{he2022masked}, MoCo-v3~\citep{chen2021empirical}). Detailed per-task results for VTAB-1K are provided in Appendix.
}
\label{tab:compare_all}
\end{table}

We evaluate our method against existing VPT-based and recent related approaches across 24 vision tasks. As shown in Table~\ref{tab:compare_all}, our DA-VPT+ consistently achieves superior performance over VPT-related methods on both supervised ViT and self-supervised backbones. On ViT-B, DA-VPT+ improves over the VPT-Deep baseline by 2.83 and 4.18 percentage points (pp) on FGVC and VTAB-1K, respectively, and outperforms E2VPT by 2.72 pp and 2.20 pp on these tasks. Notably, even without bias tuning, DA-VPT maintains strong results across major benchmarks. The improvements are particularly pronounced with self-supervised backbones, where our method also surpasses full fine-tuning on all backbones while using fewer parameters. These results highlight the effectiveness and generalizability of our approach across diverse downstream tasks compared to other VPT-based methods.

% seg results and analyze

\begin{table}[htbp]
    \centering
    \small
    \resizebox{0.95\linewidth}{!}{
    \begin{tabular}{|c|c|c|c|c|c|}
    \hline
    \multirow{2}{*}{Method} &  \multirow{2}{*}{\#Param} & \multicolumn{2}{c|}{ADE20K} & \multicolumn{2}{c|}{PASCAL Context} \\ \cline{3-6} 
                                              &          &    mIoU-SS & mIoU-Ms & mIoU-SS & mIoU-Ms \\ \hline
    Full-Tuning                               &  317.3M  & 47.60 & 49.18 & 53.69 & 55.21 \\ 
    Linear                                    &  13.1M   & 38.09 & 39.16 & 46.06 & 48.13 \\ 
    Bias                                      &  13.2M   & 43.61 & 45.73 & 45.15 & 46.47 \\ \hline
    VPT (baseline)                            &  13.6M   & 44.08 & 46.01 & 49.51 & 50.46 \\ 
    SPT-LoRA \citep{he2023sensitivity}         &  14.6M   & 45.40 & 47.50  & --  & -- \\ 
    SPT-Adapter \citep{he2023sensitivity}      &  14.6M   & 45.20 & 47.20 & -- & -- \\
    DA-VPT (ours)                             &  13.6M  & 45.10  & 47.07 & 50.15  & 51.04 \\
    DA-VPT+ (ours)                            &  13.7M   & \textbf{46.47} & \textbf{47.21} & \textbf{50.40} & \textbf{51.28} \\ \hline
    \end{tabular}
    }
    \caption{
    \small
    \textbf{Results of Semantic Segmentation on ADE20K and PASCAL Context.} We report mIoU-SS (single-scale inference) and mIoU-MS (multi-scale inference). All experiments use the \textbf{ViT-L} backbone pre-trained on ImageNet-21K. The \#Param column indicates the total number of tunable parameters in the entire framework. For SPT \citep{he2023sensitivity}, we report the results from the original paper, while for other settings and our baseline, we provide our reproduced results. We highlight the best results other than the full fine-tuning.}
    \label{tab:seg_result}
\end{table}

Table~\ref{tab:seg_result} demonstrates our proposed methods, DA-VPT and DA-VPT+, achieve significant improvements over existing baselines and recent competitive methods in semantic segmentation tasks on both the ADE20K and PASCAL Context datasets. Compared to classification tasks, dense prediction tasks such as segmentation are much more challenging. Notably, lightweight PEFT methods like Linear or Bias exhibit low efficiency compared to full fine-tuning. In such challenging tasks, our proposed DA-VPT+ still achieves comparable performance while using only 4.3\% of the tunable parameters, demonstrating both high parameter efficiency and effectiveness across both datasets.

\label{sec:sota}

\begin{table}
\centering
\resizebox{0.99\linewidth}{!}{
\begin{tabular}{|l|ccccc|cc|r|}
\toprule
\multirow{2}{*}{\diagbox{Method}{Dataset}} &CUB-200 & NABirds & Oxford  & Stanford  &  Stanford & Mean & Mean \\
 & -2011 &  & Flowers & Dogs & Cars & Acc (\%) & Params (M) \\
\midrule
Full fine-tuning \citep{jia2022visual} & 87.3 & 82.7 & 98.8 & 89.4 & 84.5 & 88.54 & 85.98 \\
Linear Probing \citep{jia2022visual} & 85.3 & 75.9 & 97.9 & 86.2 & 51.3 & 79.32 & \textbf{0.18} \\ 
\midrule
Adapter \citep{adapter19} & 87.1 & 84.3 & 98.5 & 89.8 & 68.6 & 85.67 & 0.41 \\
Bias \citep{zaken2021bitfit} & 88.4 & 84.2 & 98.8 & 91.2 & 79.4 & 88.41 & 0.28 \\
AdaptFormer \citep{adaptformer22} & 87.4 & 84.8 & 99.0 & 90.7 & 81.0 & 88.58 & 1.54 \\
VPT-Shallow \citep{jia2022visual} & 86.7 & 78.8 & 98.4 & 90.7 & 68.7 & 84.62 & 0.25 \\
VPT-Deep \citep{jia2022visual} & 88.5 & 84.2 & 99.0 & 90.2 & 83.6 & 89.11 & 0.85 \\
SSF \citep{lian2022scaling} & 89.5 & 85.7 & 99.6 & 89.6 & 89.2 & 90.72 & 0.39 \\

SNF \citep{wang2023adapting} & 90.2 & 87.4 & 99.7 & 89.5 & 86.9 & 90.74 & 0.25 \\
MP\citep{gao2023tuning} & 89.3 & 84.9& 99.6 & 89.5 & 83.6 & 89.38 & 1.20 \\
E2VPT \citep{han20232vpt} & 89.1 & 84.6 & 99.1 & 90.5 & 82.8 & 89.22 & 0.65 \\
MoSA \citep{zhang2023mosa} & 89.3 & 85.7 & 99.2 & \textbf{91.9} & 83.4 & 89.90 & 1.54 \\

\midrule
VPT (Baseline) & 88.6 & 85.7 & 99.2 & 89.0 & 87.4 & 90.14 & 0.36 \\
DA-VPT (Ours) & 90.2 & 87.4 & 99.4 & 89.4 & 89.7 & 91.22 & 0.30 \\
DA-VPT+ (Ours) & \textbf{90.8} & \textbf{88.3} & \textbf{99.8} & 89.8 & \textbf{91.0} & \textbf{91.94} & 0.32 \\
\bottomrule
\end{tabular}
}
\caption{
\small
\textbf{Comparison of various fine-tuning methods on different downstream tasks.} The ViT-B model pre-trained on ImageNet-21K is used as basic backbone. Top-1 accuracy (\%) is reported and the best result is in \textbf{bold}.
}

\label{tab:fgvc_vtag_vit}
\end{table}

Table~\ref{tab:fgvc_vtag_vit} compares state-of-the-art PEFT methods on FGVC \citep{jia2022visual} using ImageNet-21K pre-trained ViT-B. Our DA-VPT+ achieves the highest mean accuracy of 91.94\% across all datasets, surpassing previous SOTA methods SNF \citep{wang2023adapting} and MoSA \citep{zhang2023mosa} on FGVC. Notable improvements include gains of 0.6 and 1.8 percentage points on CUB and Cars datasets respectively. Both DA-VPT and DA-VPT+ outperform the VPT baseline and full fine-tuning with significant margin while using fewer parameters, demonstrating superior accuracy-efficiency trade-off compared to full fine-tuning and existing PEFT methods.

%% ========================== ablation study ============================================

\subsection{Ablation Studies and Discussion}
\label{sec:ablation}
\subsubsection{Ablation Study}

\begin{table*}[ht]

\caption{\small \textbf{Ablation study on different components in our DA-VPT on two datasets: CUB-200-2011 in FGVC and \textit{Natural} in VTAB-1k.} For each $\mathcal{L}_\text{ML}$ component, we also search for its optimal hyperparameter. The learnable [CLS] token is combined with Efficient Bias for simplicity. We fixed the number of prompts to 20 for all settings. The latency and memory are tested in the same server with RTX4090 GPU.}

\label{table:ablative_components}
\centering
%\begin{small}
%\tabcolsep=0.20cm
\resizebox{0.80\textwidth}{!}{
\begin{tabular}{|c|c|c|>{\centering\arraybackslash}p{1cm}|>{\centering\arraybackslash}p{2cm}|>{\centering\arraybackslash}p{1cm}|>{\centering\arraybackslash}p{2cm}|>{\centering\arraybackslash}p{1.5cm}|>{\centering\arraybackslash}p{1.5cm}|} 
\hline 
\multicolumn{3}{|c|}{Components of our Techniques} 
& \multicolumn{2}{c|}{VTAB-1k \textit{Natural (7)}} 
& \multicolumn{2}{c|}{FGVC CUB-200} 
& {Latency} 
& {Memory} \\  
\cline{1-7} 
$\mathcal{L}_\text{ML}(\mathbf{x}, \mathbf{p})$  & $\mathcal{L}_\text{ML}(\mathbf{p}, \mathbf{x}_\text{cls})$ & Efficient Bias & Param & Accuracy & Param & Accuracy &  (ms/img) & (GB) \\  \hline
~ 	& ~ & ~ 					& \multirow{4}{1cm}{0.14M (0.16\%)} & 79.45 (base) & \multirow{4}{1cm}{0.20M (0.24\%)} & 88.64 (base) & \B{1.41} & \B{2.41} \\
\checkmark 	& ~ & ~ 			& ~ & 79.47 (+0.02) & ~  & 89.24 (+0.60) & 1.51 & 2.41 \\
~ & 	\checkmark 	&~ 			& ~ & 79.51 (+0.06) & ~  & 89.06 (+0.42) & 1.52 & 2.41 \\
\checkmark &  \checkmark	 & ~	 & ~ & 80.53 (+1.08) & ~  & 89.86 (+1.22) & 1.54& 2.41 \\
\cline{4-7}
~ & ~	& \checkmark &  \multirow{4}{1cm}{0.16M (0.19\%)} & 80.06 (+0.61) &  \multirow{4}{1cm}{0.23M (0.27\%)} & 89.55 (+0.91) & 1.45 & 2.76 \\
~ &  \checkmark	& \checkmark 	& ~ & 81.02 (+1.57) & ~ & 90.41 (+1.77) & 1.53 & 2.76 \\
\checkmark & ~	& \checkmark 	& ~ & 81.50 (+2.05) & ~ & 90.54 (+1.90) & 1.53 & 2.76 \\
\checkmark & 	\checkmark & \checkmark & ~ & \B{81.98 (+2.53)} & ~ & \B{90.89 (+2.25)} & 1.56 & 2.76 \\
\hline
\end{tabular}
}
%\end{small}
\end{table*}

The ablation study demonstrates the individual and collective contributions of each component in our proposed DA-VPT method on the CUB dataset from the FGVC benchmark and the Natural task category from the VTAB-1k benchmark. The metric learning losses, $\mathcal{L}_\text{ML}(\mathbf{x}, \mathbf{p})$ and $\mathcal{L}_\text{ML}(\mathbf{p}, \mathbf{x}_\text{cls})$, lead to accuracy improvements of 1.08 \textit{pp} on VTAB-1k Natural and 1.22 \textit{pp} on CUB over the baseline. The integration of Efficient Bias further enhances the performance, contributing to an additional 0.97 \textit{pp} and 1.03 \textit{pp} improvement on the respective datasets. When all three components are combined, our DA-VPT method achieves the highest performance, with total accuracy improvements of 2.05 \textit{pp} on VTAB-1k Natural and 2.25 \textit{pp} on CUB.

While the incorporation of these components introduces a minimal increase in latency and memory usage, the gained accuracy far outweighs this slight trade-off. Note that the combination of $\mathcal{L}_\text{ML}(\mathbf{x}, \mathbf{p})$, $\mathcal{L}_\text{ML}(\mathbf{p}, \mathbf{x}_\text{cls})$ and Efficient Bias yields substantial improvements with only a modest increase in parameters. This highlights the efficiency of our method in achieving significant performance gains with minimal parameter overhead.

\begin{figure*}[htpb]
\centering
    \begin{subfigure}[t]{0.24\textwidth}
            \includegraphics[width=\textwidth]{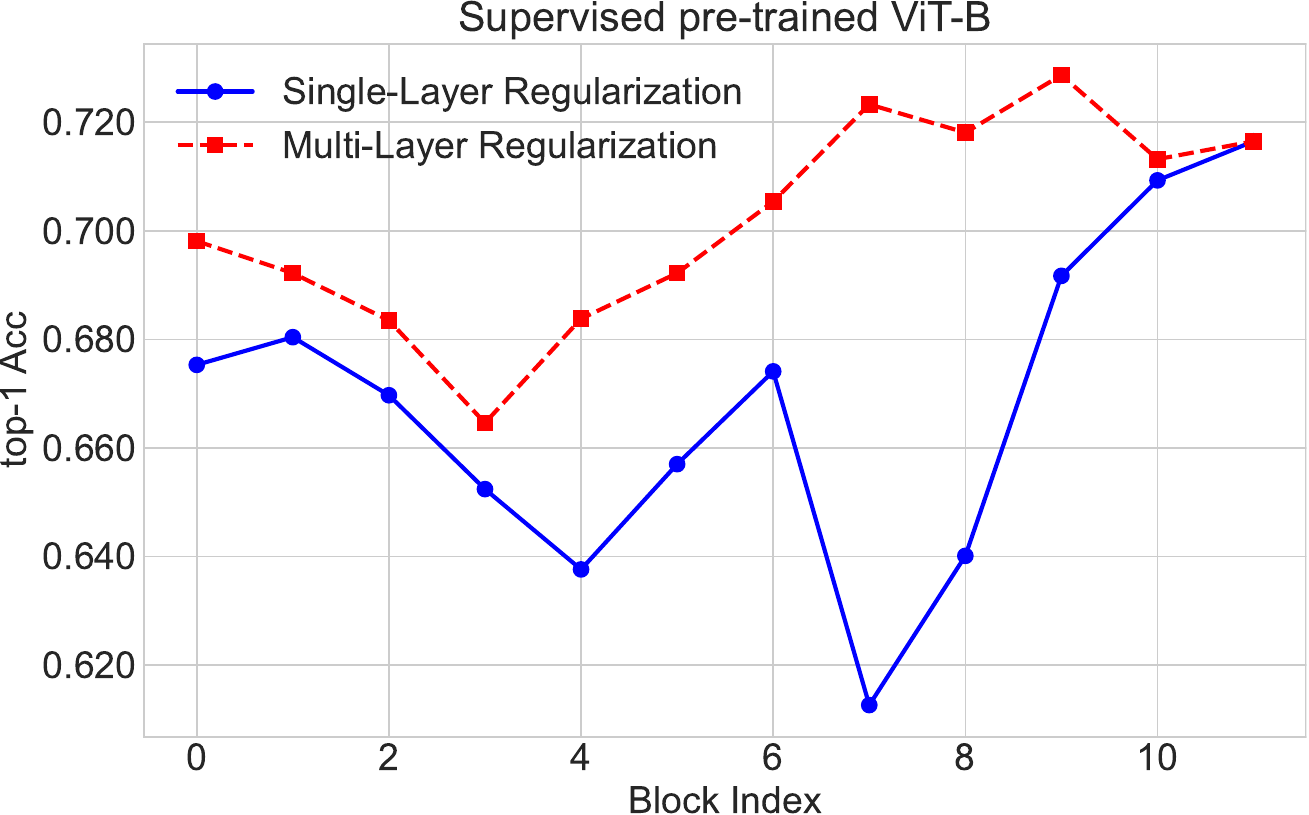}
            \caption{}
            \label{fig:impact_1:num_of_layers}
    \end{subfigure}
    \hspace{-0.5em}
    \begin{subfigure}[t]{0.24\textwidth}
        \includegraphics[width=\textwidth]{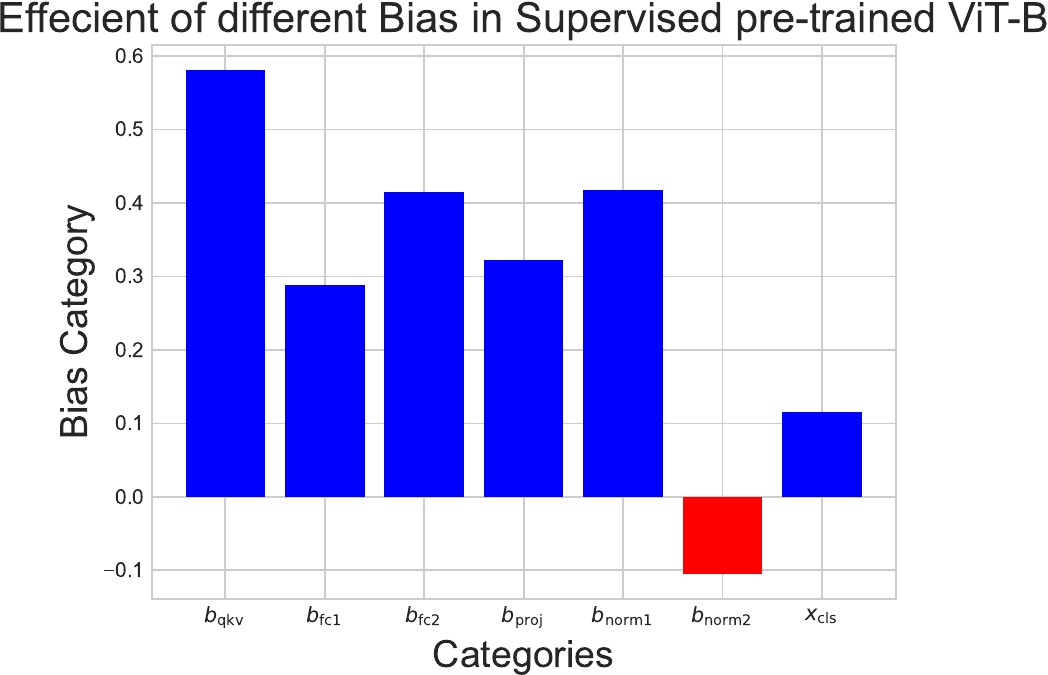}
        \caption{}
        \label{fig:impact_1:efficient_bias}
    \end{subfigure}
    \hspace{-0.5em}
    \begin{subfigure}[t]{0.24\textwidth}
        \includegraphics[width=\textwidth]{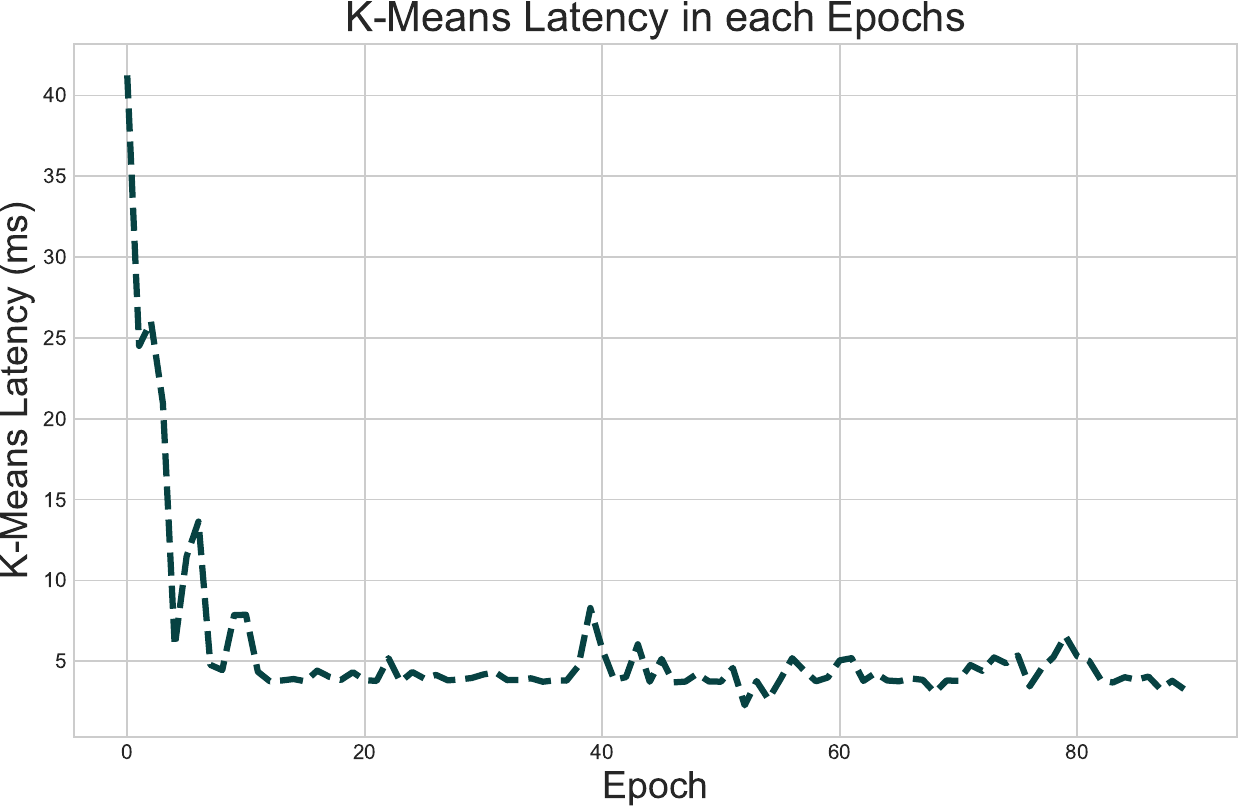}
        \caption{}
        \label{fig:impact_1:k-means}
    \end{subfigure}
    \hspace{-0.5em}
    \begin{subfigure}[t]{0.24\textwidth}
        \includegraphics[width=\textwidth]{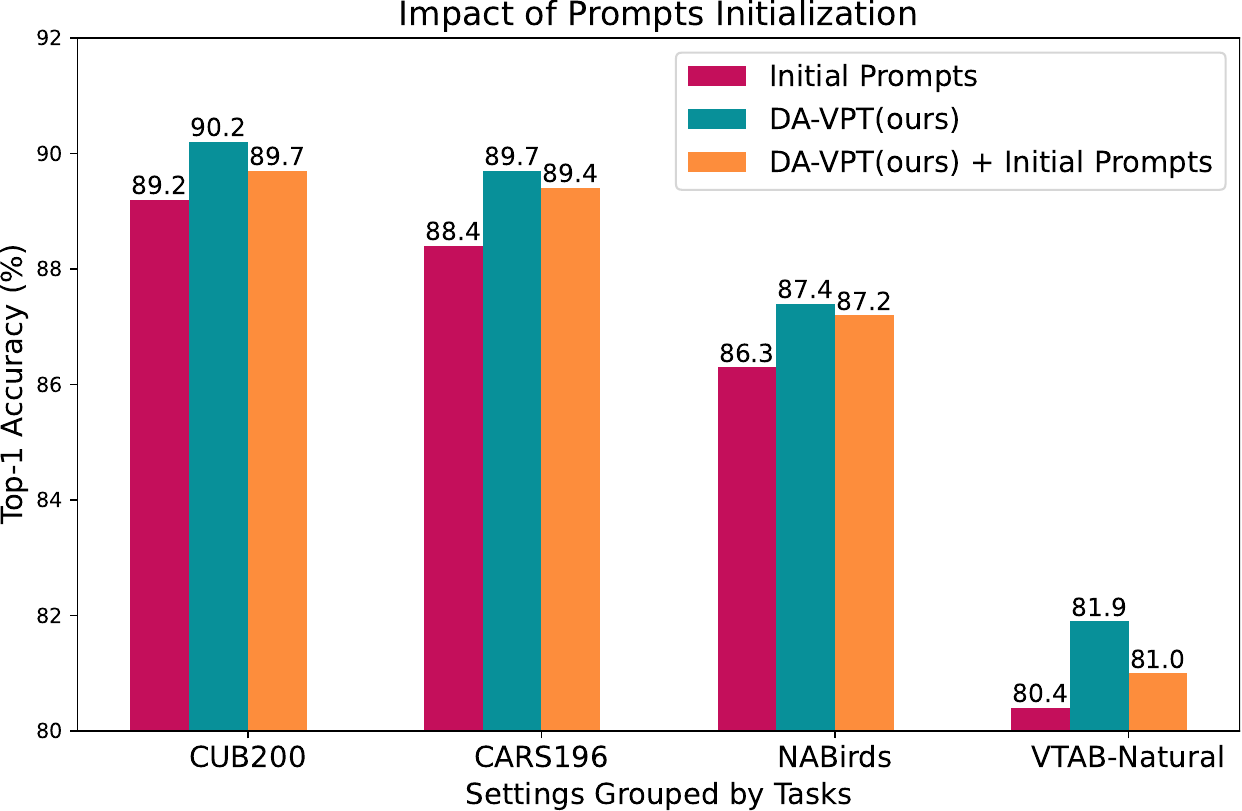}
        \caption{}
        \label{fig:impact_1:impact_of_init}
    \end{subfigure}
    
\caption{ \small  \ref{fig:impact_1:num_of_layers} Illustrates the impact of the number and position of the layers to which the proposed metric learning loss is applied. \ref{fig:impact_1:k-means} This figure shows the latency of the k-means calculation in each epoch. \ref{fig:impact_1:efficient_bias} Illustrates the importance of each category of efficient bias measured on the CUB-200-2011 dataset. \ref{fig:impact_1:impact_of_init} This figure shows the comparison of the performance with or without the prompts initialization with data mean value.
} \label{fig:impact_1}
\end{figure*}

%% ========================== parameter impaction ======================================

\subsubsection{Analysis of Parameter Impacts}
\label{sec:impacts}

\textbf{Layer-wise Impact.} We first examine the effectiveness of our metric learning loss when applied to different layers. As shown by the blue line in Figure~\ref{fig:impact_1:num_of_layers}, applying the loss to the final layer yields optimal results in most cases, likely due to the presence of higher-level semantic features in deeper layers. We further investigate the effect of applying our loss across multiple consecutive layers, represented by the red line, which shows the performance when applying the loss from a specific layer through to the final layer. The impact varies notably across different pre-trained models, with detailed results for MAE and MoCo provided in the Appendix.

\textbf{Efficient Bias Components.} Figure~\ref{fig:impact_1:efficient_bias} demonstrates that specific categories of efficient bias contribute disproportionately to performance improvements. This observation underscores the importance of selective optimization of bias components rather than uniform adjustment across all parameters.

\textbf{Semantic Mapping Updates.} The computational cost of k-means clustering for semantic mapping updates is illustrated in Figure~\ref{fig:impact_1:k-means}. Notably, while the initial epochs incur higher computational overhead, the latency decreases significantly in later epochs as class representations stabilize. This suggests that the computational cost of maintaining dynamic class-prompt mappings becomes negligible as training progresses.

\textbf{Prompt Initialization.} We investigate the impact of prompt initialization by comparing mean value initialization on both baseline VPT and our proposed DA-VPT, following the methodology of \citet{wang2024revisiting}, where prompts are initialized with \textit{mean pooling} values from the dataset at each layer. As shown in Figure~\ref{fig:impact_1:impact_of_init}, our analysis reveals that such initialization actually impedes our method's effectiveness. We attribute this to the increased difficulty in guiding prompts to capture discriminative information when initialized with homogeneous content from the mean value. Additional parameter impact analyses are provided in the Appendix.

\section{Conclusion}
\label{sec:conclusion}

This paper introduces Distribution-Aware Visual Prompt Tuning (DA-VPT), a novel framework that improves visual prompt learning in Vision Transformers (ViT) through semantic metric construction between prompts and image features. Our method guides prompts to serve as effective bridges for semantic information flow between image patches and class tokens via the attention mechanism. Extensive evaluations across 24 visual recognition and 2 segmentation tasks demonstrate that DA-VPT significantly outperforms vanilla VPT and other related methods while using fewer prompts and parameters. Our results highlight the importance of considering the intrinsic connection between visual prompts and data samples and showcase the potential of our approach to enhance the transfer learning capabilities of pre-trained vision models. We believe that our findings can inspire further research on parameter-efficient fine-tuning strategies and contribute to the development of more effective and efficient vision foundation models.

{
    \small
    \bibliographystyle{ieeenat_fullname}
    \bibliography{main}
}

\newpage
\section*{Appendix}

\thispagestyle{empty}

\appendix

\section{Details About the Experiments}
\label{exp:details}

\subsection{Datasets}
\paragraph{Classification Datasets.} \textbf{FGVC} encompasses five fine-grained visual classification datasets: CUB-200-2011 \citep{wah2011caltech}, NABirds \citep{van2015building}, Oxford Flowers \citep{nilsback2008automated}, Stanford Dogs \citep{khosla2011novel}, and Stanford Cars \citep{gebru2017fine}. Following \citet{jia2022visual}, we split each dataset into \texttt{train} (90\%) and \texttt{val} (10\%) subsets. 

\textbf{VTAB-1K} contains 19 diverse visual tasks across three categories: (i) \textit{Natural} tasks involving standard camera images for object classification and scene recognition; (ii) \textit{Specialized} tasks using domain-specific imagery such as medical scans and satellite data; and (iii) \textit{Structured} tasks focusing on spatial relationships and object properties.

\paragraph{Segmentation Datasets.} We evaluate on two semantic segmentation benchmarks: \textbf{ADE20K} with 150 fine-grained semantic concepts, and \textbf{PASCAL Context} providing pixel-wise annotations across 60 object classes. For dataset partitioning, we strictly follow the protocol established in VPT \citep{jia2022visual}. Complete dataset statistics and task details are provided in Table~\ref{tab:data_info}.

\subsection{Implementation Details}
\paragraph{Classification Tasks.} For FGVC datasets, we employ standard data augmentation: random resizing and cropping to 224$\times$224 pixels with random horizontal flipping. For VTAB-1K, following \citet{zhai2019large} and \citet{jia2022visual}, images are directly resized to 224$\times$224 pixels without additional augmentation.

Model training utilizes the AdamW optimizer with a batch size of 32 over 100 epochs. The learning rate follows a combined schedule: a 10-epoch linear warm-up followed by cosine decay \citep{loshchilov2016sgdr} from the initial value to 1e-8. We determine optimal hyperparameters through cross-validation on the validation set. Following established protocols \citep{jia2022visual,lian2022scaling,gao2023tuning}, we report mean accuracy across three runs with different random seeds.

\paragraph{Segmentation Tasks.} We implement our experiments using the SETR framework \citep{zheng2021rethinking} through MMSegmentation. We adopt the SETR-PUP configuration, utilizing one primary head and three auxiliary heads to process features from transformer layers 9, 12, 18, and 24. Training follows \citet{zheng2021rethinking}: 160k iterations for ADE20K and 80k iterations for PASCAL Context, with hyperparameter optimization mirroring our classification approach.

For multi-class segmentation samples, we adapt the class assignment strategy by randomly selecting one non-background class as the target class for visual prompt assignment during each iteration, accounting for the pixel-wise multi-class nature of segmentation tasks.

\subsection{Hyperparameter Configuration}
\paragraph{Parameter Search Space.} Table \ref{tab:param_space} details our hyperparameter search space for each task, including learning rate, weight decay, and the number and location of layers guided by semantic metrics loss. For the main results, we maintain default values for Proxy-Anchor loss parameters to narrow the parameter searching space.
%%%%%%%%%%%%%%%%% searching hyper-parameters %%%%%%%%%%%%%%%%%%s
\begin{table}[ht]
\centering
\resizebox{0.99\linewidth}{!}{
    \begin{tabular}{l l}
    \hline \B{Configuration} & \B{Value} \\ \hline
    Optimizer & AdamW \cite{loshchilov2017decoupled} \\
    Base learning rate range & \{1e-3, 5e-4, 1e-4, 5e-5\} \\
    Weight decay range & \{0.001, 0.005, 0.01, 0.05, 0.1, 0.5, 1.0\} \\
    Learning rate schedule & Cosine Decay \cite{loshchilov2016sgdr} \\
    Layers applied guidance & \{ 12, 10, 8, 6, 4, 2, 0\}\\
    Num of prompts applied guidance & \{ 5, 10, 20, 40\} \\
    Proxy-Anchor $\delta$ & 32.0 \\
    Proxy-Anchor $\tau$ & 10.0 \\
    Batch size & 32 \\
    Warmup epoch & 10 \\
    Total epoch & 100 (ViT-B/16) \\
    Augmentation & RandomResizedCrop, RandomHorizontalFlip \\
    \hline
    \end{tabular}
}
\caption{Hyper Parameters Searching Space and Training configuration in our experiments}
\label{tab:param_space}
\end{table}

\paragraph{Prompt Configuration for Tasks with Limited Classes.} For tasks with limited classes ($C<5$), we augment the visual prompts by adding extra unassigned prompts that are not guided by semantic metrics loss. This approach empirically improves prompt transferability, particularly in tasks with very few classes (e.g., Patch Camelyon, Retinopathy, or KITTI-Dist in VTAB). The extra prompts help maintain an effective number of visual prompts in the guiding layer while preserving the semantic structure of the original class-assigned prompts.

%%%%%%%%%%%%%%%%%%%%%%%%%%%%%%%%%%% details of datasets %%%%%%%%%%%%%%%%%%%%%%%%%%%%%%%%%%%%%%%

\begin{table*}[ht]
\centering
\small
\resizebox{0.95\linewidth}{!}{
\begin{tabular}{|c|c|c|c|c|c|}
\hline  \B{Datasets} & \B{Task Description} & \B{Classes} & \B{Train Size} & \B{Val Size} & \B{Test Size} \\ 
	
\hline  \multicolumn{6}{|c|}{Fine-Grained Visual Classification (FGVC) \citep{jia2022visual} } \\  \hline 

CUB-200-2011 \citep{wah2011caltech} & Fine-grained Bird Species Recognition & 200 & 5,394 & 600 & 5,794 \\ 
NABirds \citep{van2015building} & Fine-grained Bird Species Recognition & 55 & 21,536 & 2,393 & 24,633 \\ 
Oxford Flowers \citep{nilsback2008automated} & Fine-Grained Flower Species recognition & 102 & 1,020 & 1,020 & 6,149 \\  
Stanford Dogs \citep{khosla2011novel} & Fine-grained Dog Species Recognition & 120 & 10,800 & 1,200 & 8,580 \\ 
Stanford Cars \citep{gebru2017fine} & Fine-grained Car Classification & 196 & 7,329 & 815 & 8,041 \\ 

\hline  \multicolumn{6}{|c|}{Visual Task Adaptation Benchmark (VTAB-1k) \citep{zhai2019large} } \\  \hline

Caltech101 \citep{fei2006one} & \multirow{7}{5cm}{ Natural-Tasks (7) \newline Natural images captured using standard cameras.} & 102 & \multirow{7}{*}{800/1000} & \multirow{7}{*}{200} & 6,084 \\ 
CIFAR-100 \citep{krizhevsky2009learning} & & 100 &  & & 10,000 \\ 
DTD \citep{cimpoi2014describing} &  & 47 &  &  & 1,880 \\ 
Oxford-Flowers102 \citep{nilsback2006visual} &  & 102 &  &  & 6,149 \\ 
Oxford-PetS \citep{parkhi2012cats} &  & 37 &  &  & 3,669 \\ 
SVHN \citep{netzer2011reading} &  & 10 &  &  & 26,032 \\ 
Sun397 \citep{xiao2010sun} &  & 397 &  &  & 21,750 \\ 
\hline 
Patch Camelyon \citep{veeling2018rotation} &  \multirow{4}{5cm}{Special-Tasks (4) \newline Images captured via specialized equipments } & 2 & \multirow{4}{*}{800/1000} & \multirow{4}{*}{200} & 32,768 \\ 
EuroSAT \citep{helber2019eurosat} &  & 10 &  &  & 5,400 \\ 
Resisc45 \citep{cheng2017remote} &  & 45 &  &  & 1,880 \\
Retinopathy \citep{retinopathy} & & 5 & & & 42,670 \\
\hline 
Clevr/count \citep{johnson2017clevr} & \multirow{8}{5cm}{Structured-Tasks (8) \newline Require geometric comprehension} &  & \multirow{8}{*}{800/1000} & \multirow{8}{*}{200} & 15,000  \\ 
Clevr/distance \citep{johnson2017clevr} &  & 6 & & & 15,000 \\ 
DMLab \citep{beattie2016deepmind} &  &  6 &  &  & 22,735 \\ 
KITTI-Dist \citep{geiger2013vision} &  & 4 &  &  & 711 \\ 
dSprites/location \citep{matthey2017dsprites} &  & 16 &  &  & 73,728 \\ 
dSprites/orientation \citep{matthey2017dsprites} &  & 16 &  &  & 73,728 \\ 
SmallNORB/azimuth \citep{lecun2004learning}  &  & 18 &  &  & 12,150 \\ 
SmallNORB/elevation \citep{lecun2004learning} &  & 18 &  &  & 12,150 \\ 

\hline  \multicolumn{6}{|c|}{Image Semantic Segmentation } \\  \hline
ADE20K \citep{zhou2019semantic} & \multirow{2}{5cm}{Fine-grained images with pixel-wise \newline semantic annotations} & 150 & 20210 & 2000 & 3352 \\
PASCAL Context \citep{mottaghi2014role} & & 60 &  4998 & 5105 & ---- \\

\hline

\end{tabular} 
}
\caption{ \textbf{The details and specifications of the downstream task datasets we selected to evaluate our proposed framework.}}
\label{tab:data_info}
\end{table*}

\section{Discussion and Comparison with Self-SPT}
\label{sec:compare_self_spt}

\paragraph{Methodological Distinctions.} While both Self-SPT and our work leverage prompt distributions to enhance representation learning, they differ fundamentally in their approaches. Self-SPT attempts to align prompt and visual token distributions through initialization, using mean or max pooling of input data to set background values. In contrast, our method achieves distribution matching at the semantic level throughout the optimization process. This semantic-level matching enables our visual prompts to capture discriminative features by explicitly considering class relationships during training.

\paragraph{Our Key Advantages.} Our approach demonstrates several significant advantages over Self-SPT:
\begin{itemize}
    \item \textbf{Continuous Optimization:} Our method maintains distribution regularization throughout the entire optimization process, while Self-SPT only applies distribution alignment during initialization.
    
    \item \textbf{Discriminative Feature Learning:} Through metric guidance, our prompts explicitly capture class-specific discriminative features by comparing tokens from the same and different classes. In contrast, Self-SPT's uniform background value initialization does not differentiate between class-specific features.
    
    \item \textbf{Computational Efficiency:} Our method significantly reduces pre-processing overhead by clustering class representations rather than entire visual token sets, avoiding the computational burden of Self-SPT's k-means clustering approach on the full dataset.
\end{itemize}

\paragraph{Empirical Analysis.} To thoroughly evaluate the relationship between background value initialization and metric guidance learning, we attempted to reproduce Self-SPT's results and assess its performance when integrated with our method. As illustrated in Figure~\ref{fig:compare_spt}, our experiments revealed two key findings: (1) we were unable to reproduce the performance metrics reported in the original Self-SPT paper, and (2) the background value initialization strategy did not yield measurable improvements when combined with our metric guidance approach. These results further validate our focus on semantic-level distribution matching as the primary mechanism for improving prompt optimization.

\begin{figure}
    \centering
    \includegraphics[width=0.95\linewidth]{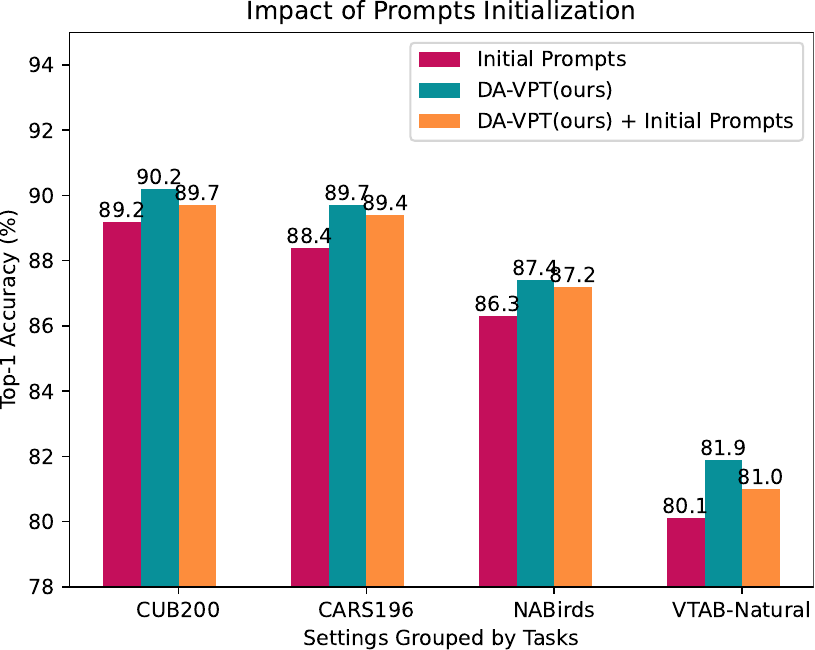}
    \caption{\textbf{Impact of Prompt Initialization Strategies.} Comparative analysis of model performance under different prompt initialization schemes. Results demonstrate that the background value initialization method proposed in Self-SPT \citep{wang2024revisiting}, which uses mean pooled visual tokens, shows no significant performance gains when combined with our distribution-aware guidance approach. This suggests that our method's effectiveness stems from its continuous distribution matching during training rather than initial prompt configurations.}
    \label{fig:compare_spt}
\end{figure}

The comparative analysis demonstrates that while both methods address prompt distribution optimization, our approach offers \textit{more robust theoretical foundations}, \textit{better computational efficiency}, and \textit{superior empirical performance}. The continuous nature of our distribution regularization, combined with explicit semantic guidance, provides a more principled framework for learning discriminative prompt representations.

\section{DA-VPT+: Integration with Bias Tuning}
\label{sec:davpt_plus}

This section examines the synergy between our proposed metric learning guidance and bias tuning in PEFT models. Our investigation is motivated by an intriguing observation from the original VPT work \citep{jia2022visual}, which reported that bias tuning adversely affects vanilla VPT optimization. We present a novel perspective on this interaction and demonstrate how our approach effectively addresses these limitations.

\paragraph{Theoretical Motivation.} The distribution of visual tokens in transformer layers is inherently constrained by the representations from previous layers and their associated visual prompts. We hypothesize that unfreezing bias terms, particularly in Key and Value projections, introduces additional flexibility in token representation. This flexibility becomes especially significant when combined with our metric learning guidance, as it allows for more nuanced distribution alignment between visual prompts and tokens.

\paragraph{Advantages over Vanilla VPT.} Unlike vanilla VPT, DA-VPT explicitly manages distribution alignment between visual prompts and tokens through metric learning guidance. This explicit alignment makes our method more responsive to distribution shifts in visual tokens, which are substantially influenced by projection layer bias terms. By simultaneously optimizing bias terms and maintaining distribution alignment, DA-VPT+ achieves more robust and effective feature representations. This integration of bias tuning with DA-VPT demonstrates how our method's distribution-aware approach can transform a previously problematic component (bias tuning) into a complementary enhancement.

\section{Supplemental Empirical Studies}
\label{sec:supp_studies}

We conduct additional empirical investigations to explore potential extensions of our proposed framework in two key directions: alternative metric learning approaches and modified architectural connections.

\paragraph{Alternative Metric Learning Losses.} We investigate the effectiveness of different metric learning losses for guiding visual prompt distributions. Specifically, we compare our proposed Proxy-Anchor (PA) loss against two alternatives: vanilla Proxy-NCA loss and triplet loss. For these vanilla losses, we treat the selected prompts as individual data points, with class assignments determined by the current training epoch. This comparison helps us understand the relative advantages of our PA loss formulation in the context of prompt optimization. Details are listed in Table \ref{tab:empirical_study}.

\paragraph{Modified Connection Structures.} We examine two architectural modifications to the base framework:
1) Cross-layer prompt connections (DA-VPT+Conn), which enable information flow between prompts at different layers, and
2) Learnable gated connections following GateVPT \citep{yoo2023improving} (DA-VPT+Gate), which introduce adaptive control over prompt interactions.

These architectural studies provide insights into the role of prompt connectivity in our distribution-aware framework. The experimental results and detailed analysis of these variations are presented in Table \ref{tab:empirical_study}.

\begin{table}
    \centering
    \resizebox{0.99\linewidth}{!}{
    \begin{tabular}{ccccc}
       Methods  & CUB200 & Cars & NABirds & VTAB-Natural \\ \hline
       VPT (baseline)     & 88.6 & 87.4 & 85.7 & 78.48 \\ \hline
       DA-VPT   & 90.2 (+1.6) & 89.7 (+2.3) & 87.4 (+1.7) & 80.25 (+1.77) \\ 
       DA-VPT+Conn  & 89.6 (+1.0) & 88.1 (+0.7) & 86.8 (+1.1) & 79.11 (+0.63) \\ 
       DA-VPT+Gate  & 89.8 (+1.2) & 88.4 (+1.0) & 87.0 (+1.3) & 79.48 (+1.00) \\ 
       DA-VPT (PNCA) & 89.2 (+0.6) & 88.2 (+0.8) & 86.9 (+1.2) & 79.22 (+0.74) \\ 
       DA-VPT (triplet) & 87.9 (-0.7) & 87.1 (-0.3) & 85.4 (-0.3) & 78.61 (+0.13) \\ \hline
    \end{tabular}
    }
    \caption{\textbf{Analysis of Connection Structures and Metric Learning Variants.} Empirical evaluation of (a) different visual prompt connection architectures across transformer layers and (b) alternative metric learning approaches. Our results indicate that neither cross-layer connections nor learnable activation gates provide substantial improvements over our base method. Furthermore, experiments with alternative metric learning losses show less stable fine-tuning performance compared to our approach, with some variants performing below the VPT baseline. These findings suggest that the effectiveness of our method primarily stems from its distribution-aware prompt optimization rather than architectural modifications or alternative metric formulations.}
    \label{tab:empirical_study}
\end{table}

\begin{algorithm}
\caption{Distribution Aware Visual Prompt Tuning (DA-VPT)}
\label{alg:da-vpt}
    \begin{algorithmic}
    \State \textbf{Input:} Pre-trained ViT model $f_\theta$, Dataset $\mathcal{D} = {(x_i, y_i)}_{i=1}^N$, \\
    number of prompts $M$, $\beta$, $\lambda$, learning rate and other related hyperparameters
    \State \textbf{Output:} Fine-tuned ViT model $f_{\theta}$
    \State Initialize $M$ prompts $\mathbf{p}^l$ for each layer $l$
    \State Get class tokens $\mathbf{S} \in \mathbb{R}^{C \times D}$ by Forward passing $f_\theta$
    \State Create a mapping from $C$ classes to $M$ prompts ($C \rightarrow M$) using k-means clustering on $\mathbf{S}$
    \While {\I{stop criteria is not satisfied}}
    \State Obtain a batch $\{x_i, y_i\}_{i=1}^n$ from $\mathcal{D}$
    \State Forward pass $\mathbf{x}_i$ through ViT $f_\theta$ with prompts $\mathbf{p}^l$
    \State Select saliency patch $\mathbf{x}$ right after attention layer in last selected blocks
    \State Calculate metric learning losses $\mathcal{L}_\text{ML}(\mathbf{x}, \mathbf{p})$ and  $\mathcal{L}_\text{ML}(\mathbf{p}, \mathbf{x}_\text{cls})$
    \State Calculate cross-entropy loss $\mathcal{L}_\text{CE}$
    \State Minimize loss: $\mathcal{L} = \mathcal{L}_\text{CE} + \beta \mathcal{L}_\text{ML}(\mathbf{x}, \mathbf{p}) + \lambda \mathcal{L}_\text{ML}(\mathbf{p}, \mathbf{x}_\text{cls})$
    \State Update $\mathbf{p}$ and other learnable parameters from Backward of $\mathcal{L}$
    \State Update class tokens $\mathbf{S}$ and class-prompt mapping $C \rightarrow M$ after certain steps
    \EndWhile
    \State \textbf{return} Fine-tuned ViT model $f_{\theta}$
    \end{algorithmic}
\end{algorithm}

\section{The Proof and Detial of theorem 1}

\begin{theorem}
For a weight perturbation $\Delta a_i$ calculated using the softmax function, there is an approximate relationship:
\[
\Delta a_i \approx a_i (1 - a_i) \Delta s_i,
\]
where $\Delta s_i$ is a small change in the attention score $s_i$, and
\[
\Delta s_i = \frac{\Delta p^\top v_i}{\sqrt{d}}.
\]
\end{theorem}

\begin{proof}

The attention weights \(a_i\) are calculated using the softmax function applied to the attention scores \(s_i\):
   \[
   a_i = \frac{e^{s_i}}{\sum_{j} e^{s_j}}.
   \]

The partial derivative of \(a_i\) with respect to \(s_j\) is given by:
   \[
   \frac{\partial a_i}{\partial s_j} = a_i (\delta_{ij} - a_j),
   \]
where \(\delta_{ij}\) is the Kronecker delta function:
   \[
   \delta_{ij} = \begin{cases}
   1, & \text{if } i = j, \\
   0, & \text{if } i \ne j.
   \end{cases}
   \]

For small perturbations \(\Delta s_j\), we can approximate the change in \(a_i\) using a first-order Taylor expansion:
   \[
   \Delta a_i \approx \sum_{j} \frac{\partial a_i}{\partial s_j} \Delta s_j.
   \]

Then we substitute the expression for the derivative:
   \[
   \Delta a_i \approx \sum_{j} a_i (\delta_{ij} - a_j) \Delta s_j.
   \]

Split the summation into two parts:
   \[
   \Delta a_i \approx a_i (1 - a_i) \Delta s_i - a_i \sum_{j \ne i} a_j \Delta s_j.
   \]

We assume that the weighted sum of the perturbations \(\Delta s_j\) for \(j \ne i\) is negligible:
   \[
   \sum_{j \ne i} a_j \Delta s_j \approx 0.
   \]
This approximation is reasonable when:
   \begin{itemize}
     \item The perturbations \(\Delta s_j\) for \(j \ne i\) are small and uncorrelated, so they average out.
     \item The attention weights \(a_j\) for \(j \ne i\) are small (i.e., \(a_i\) is dominant).
   \end{itemize}

Under this assumption, the expression simplifies to:
   \[
   \Delta a_i \approx a_i (1 - a_i) \Delta s_i.
   \]

\end{proof}
 %%%% ================================ Limitation =============================================
\section{Limitations}
\label{sec:limits}

Our Parameter-Efficient Fine-Tuning (PEFT) approach, while effective, still faces a few key challenges.

\paragraph{Hyperparameter Sensitivity.} The introduction of metric learning losses alongside the standard cross-entropy loss creates additional complexity in hyperparameter optimization. The performance of our method depends significantly on the weight ratios $\beta$ and $\lambda$, which require careful tuning for each combination of backbone model and downstream task. This dependency can make the optimization process more time-intensive compared to simpler PEFT approaches.

\paragraph{Computational Overhead.} Our method introduces additional computational costs through the metric learning losses and their associated operations. While the increased latency remains within practical bounds (typically 5\% higher than baseline PEFT methods), it may impact applications with strict real-time requirements or resource constraints.

\paragraph{Limited on some rare scenes} 
Our proposed method is proper for tasks that have a comparable number of classes. In some special tasks that have a limited number of classes (e.g., Patch Camelyon, Retinopathy, or KITTI-Dist in VTAB-1k), the number of prompts should be the same as the classes that are too few to have enough transfer capacities. In this case, we suggest adding a few supplemental visual prompts that are not assigned to class labels and are guided by the metric learning loss. In other words, we still propose to add normal visual prompts in addition to the guided prompts when the number of classes is very limited.

\section{Future Works} 

To address these limitations, we plan to develop automated hyperparameter optimization strategies, potentially leveraging meta-learning or Bayesian optimization techniques. We will also investigate more efficient metric learning formulations that maintain performance while reducing computational overhead. Additionally, our research will explore hardware-specific optimizations to minimize the latency impact in practical deployments.

Despite these challenges, our experimental results demonstrate that the performance improvements offered by our method consistently outweigh its limitations across diverse tasks and model architectures.

\section{Broader Impact}
\label{sec:impact}

Distribution Aware Visual Prompt Tuning (DA-VPT) has significant implications for both technical advancement and societal applications.

\paragraph{Technical Contributions.} Our method advances the performance in parameter-efficient fine-tuning by enabling more efficient adaptation of large vision models to specific domains. The reduced computational requirements for model specialization, coupled with improved performance on fine-grained visual tasks, make sophisticated vision models more accessible and practical for real-world applications.

\paragraph{Potential Applications.} DA-VPT could enable significant advances in several high-impact domains. In healthcare, it can facilitate more accurate medical image analysis with limited training data, potentially improving diagnostic accuracy and treatment planning. Environmental protection efforts could benefit from enhanced wildlife monitoring and biodiversity assessment capabilities. The method's efficiency also enables deployment of sophisticated vision models on edge devices, advancing assistive technologies for accessibility applications. Furthermore, industrial applications such as quality control and visual inspection systems could see substantial improvements in accuracy and reliability.

\begin{figure*}[!th]
\begin{center}
    \begin{subfigure}[b]{0.3\textwidth}
        \centering
        \includegraphics[width=\textwidth]{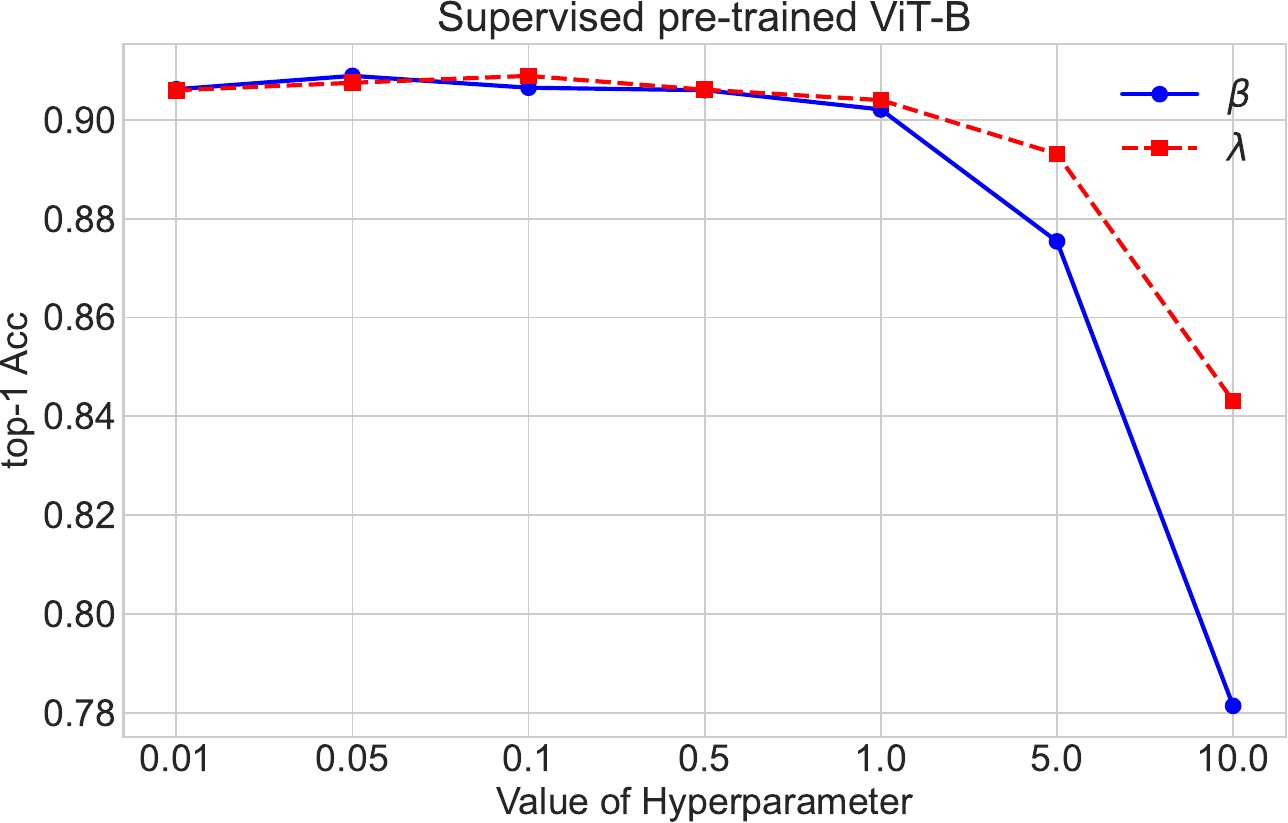}
        \label{fig:sub1}
    \end{subfigure}
    \hfill
    \begin{subfigure}[b]{0.3\textwidth}
        \centering
        \includegraphics[width=\textwidth]{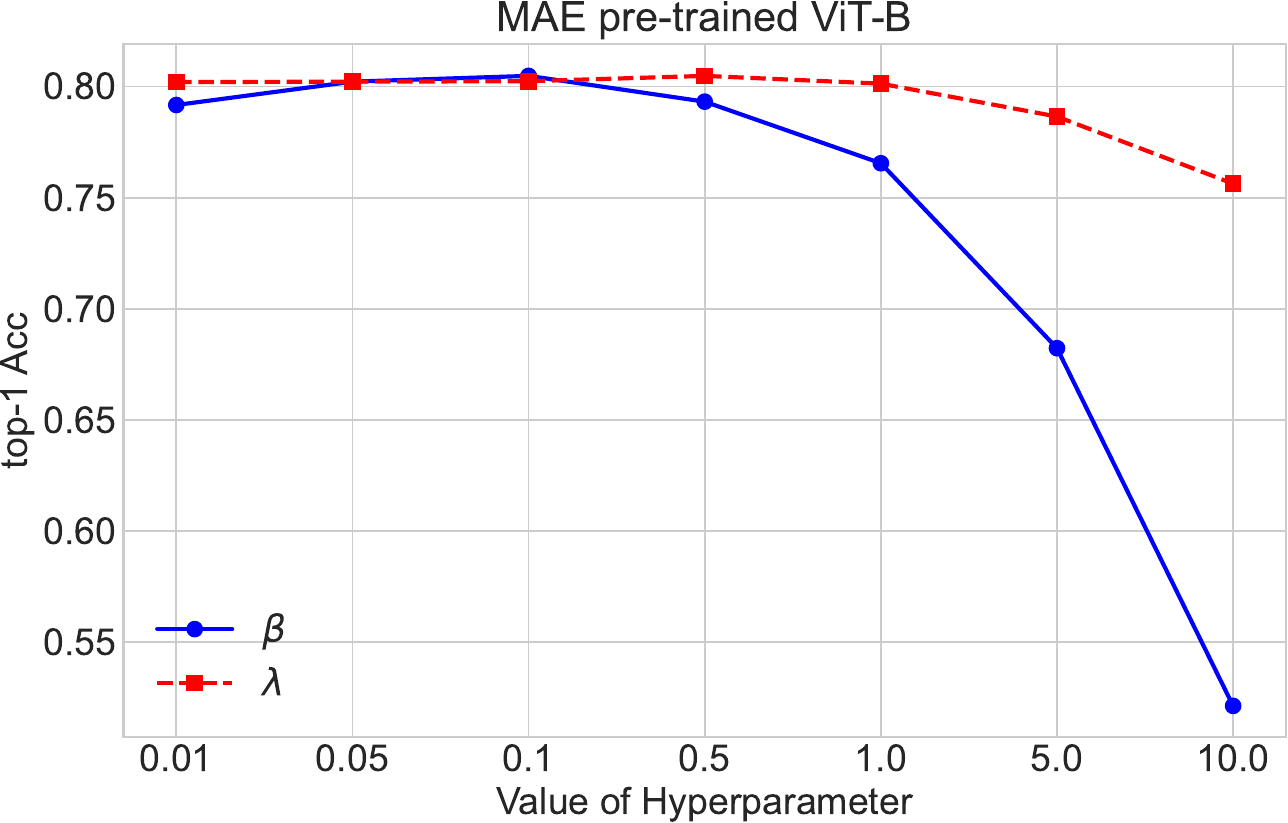}
        \label{fig:sub2}
    \end{subfigure}
    \hfill
    \begin{subfigure}[b]{0.3\textwidth}
        \centering
        \includegraphics[width=\textwidth]{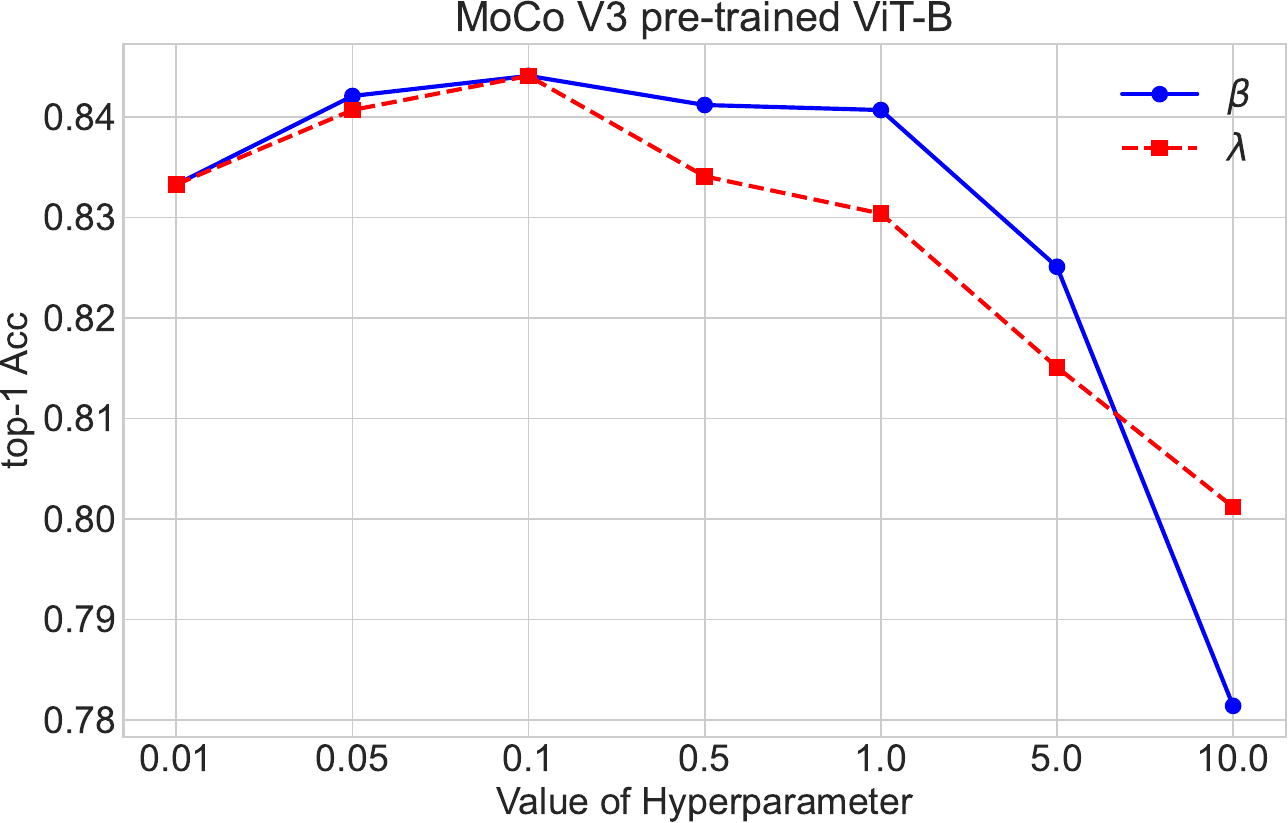}
        \label{fig:sub3}
    \end{subfigure}
    
    \vskip\baselineskip
    
    \begin{subfigure}[b]{0.3\textwidth}
        \centering
        \includegraphics[width=\textwidth]{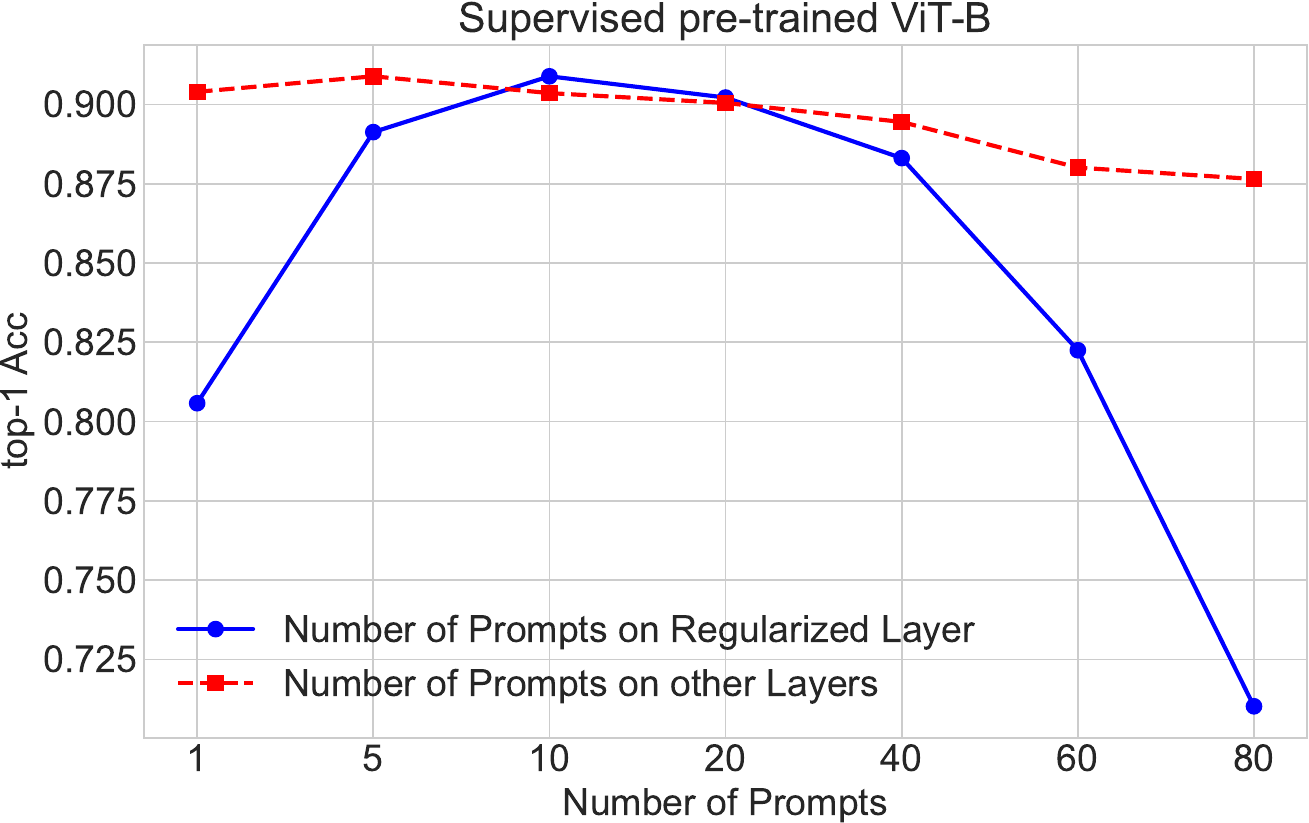}
        \label{fig:sub4}
    \end{subfigure}
    \hfill
    \begin{subfigure}[b]{0.3\textwidth}
        \centering
        \includegraphics[width=\textwidth]{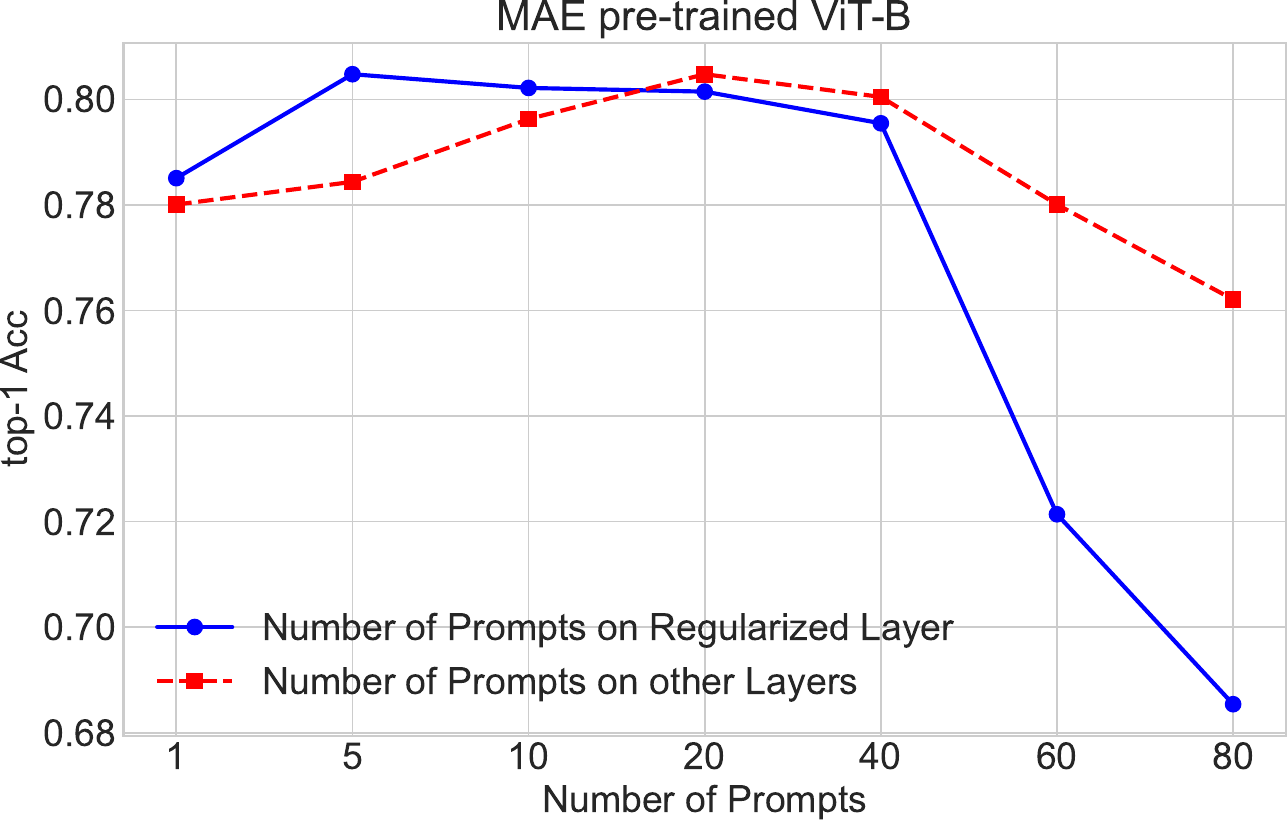}
        \label{fig:sub5}
    \end{subfigure}
    \hfill
    \begin{subfigure}[b]{0.3\textwidth}
        \centering
        \includegraphics[width=\textwidth]{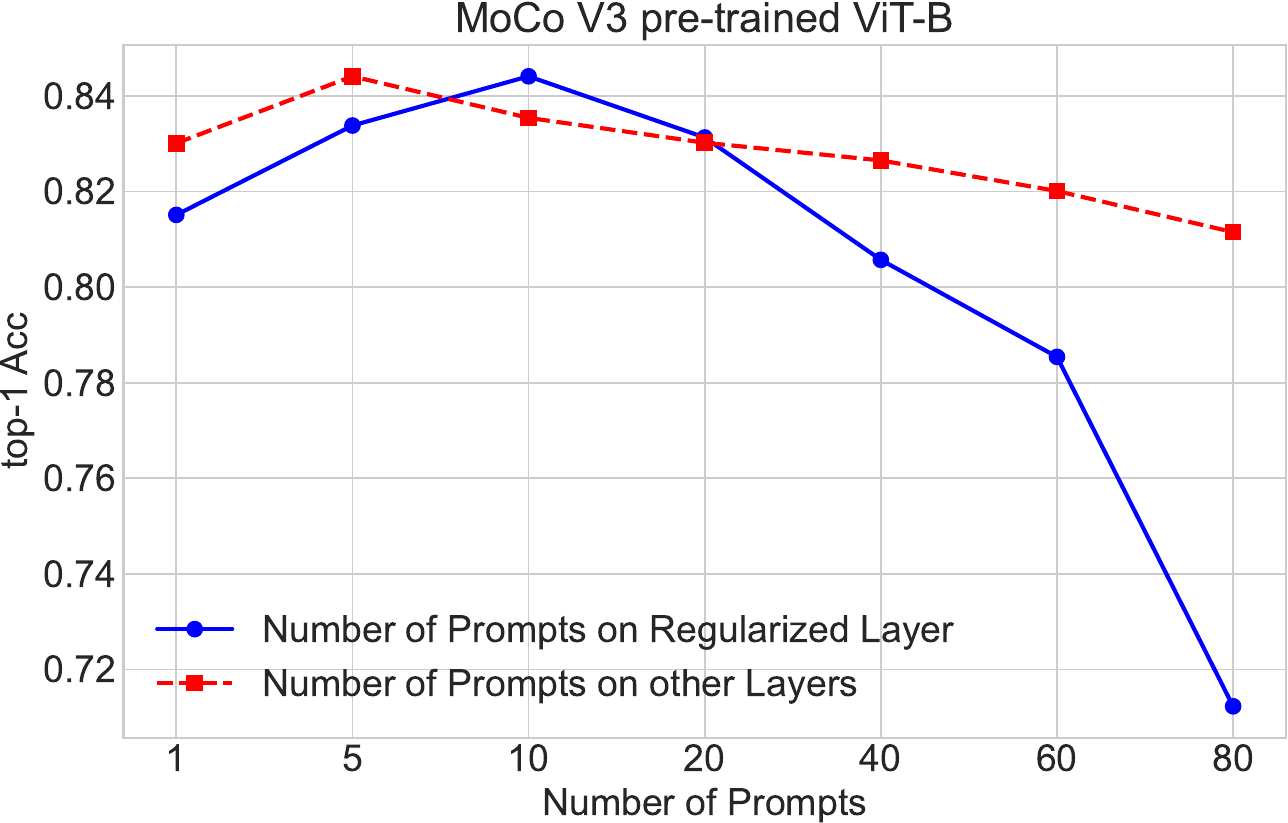}
        \label{fig:sub6}
    \end{subfigure}
    
\end{center}\vspace{-1.5em}
\caption{\small \textbf{Impact of Hyperparameters in Three Pre-trained Models on CUB-200-2011}: This figure illustrates the impact of hyperparameters on the performance of our proposed method across three pre-trained models (Supervised ViT, MAE, and MoCo-v3) on the CUB-200-2011 dataset. The hyperparameters investigated include the weight factors $\beta$ and $\lambda$ for the two proposed $\mathcal{L}_\text{ML}$ losses, the number of prompts in the metric guidance layer, and the number of prompts in other layers. The results show that the optimal weight factors are less than $1.0$, indicating that a balanced contribution from the $\mathcal{L}_\text{ML}$ losses is beneficial for performance. Furthermore, the number of prompts in the guidance layer exhibits higher sensitivity compared to the number of prompts in other layers, suggesting that the choice of prompt configuration in the guidance layer plays a crucial role in the effectiveness of our method. These findings provide insights into the importance of carefully tuning the hyperparameters to achieve optimal performance across different pre-trained models.
}
\label{fig:impact_2}
\vspace{-1em}
\end{figure*}
\begin{table}[th]
\caption{Summary of notation used throughout the paper.}
\label{tab:notation}
\centering
\resizebox{0.48\textwidth}{!}{
    \begin{tabular}{llp{6cm}}
    \toprule
    \textbf{Symbol} & \textbf{Domain} & \textbf{Description} \\
    \midrule
    $\mathbf{I}$ & $\mathbb{R}^{H \times W \times C}$ & Input image \\
    $N$ & $\mathbb{N}$ & Number of image patches \\
    $D$ & $\mathbb{N}$ & Dimension of embedding space \\
    $L$ & $\mathbb{N}$ & Number of Transformer layers \\
    $l$ & $\{1,\ldots,L\}$ & Layer index \\
    $\mathbf{x}_\text{cls}$ & $\mathbb{R}^D$ & Class [CLS] token \\
    $\mathbf{X}$ & $\mathbb{R}^{(N+1) \times D}$ & Sequence of embeddings \\
    $\mathbf{H}_i$ & $\mathbb{R}^{D \times D}$ & Attention head $i$ \\
    $\mathbf{Q}, \mathbf{K}, \mathbf{V}$ & $\mathbb{R}^{N \times D}$ & Query, Key, Value matrices \\
    $M$ & $\mathbb{N}$ & Number of prompt tokens \\
    $\mathbf{P}$ & $\mathbb{R}^{M \times D}$ & Set of prompt tokens \\
    $\mathbf{p}_k$ & $\mathbb{R}^D$ & $k$-th prompt token \\
    $\hat{\mathbf{x}}$ & $\mathbb{R}^D$ & L2-normalized vector of $\mathbf{x}$ \\
    $y_i$ & $\{1,\ldots,C\}$ & Class label for sample $i$ \\
    $C$ & $\mathbb{N}$ & Number of classes \\
    $\delta$ & $\mathbb{R}^+$ & Margin in metric learning \\
    $\tau$ & $\mathbb{R}^+$ & Temperature parameter \\
    $\mathcal{P}$ & - & Set of all prompts \\
    $\mathcal{P}^+$ & - & Set of positive prompts \\
    $\mathcal{X}_p^+$ & - & Set of positive visual tokens \\
    $\mathcal{X}_p^-$ & - & Set of negative visual tokens \\
    $\beta, \lambda$ & $\mathbb{R}^+$ & Loss weighting hyperparameters \\
    $\mathbf{S}$ & $\mathbb{R}^{C \times D}$ & Class representations \\
    $\mathbf{W}_Q^l$ & $\mathbb{R}^{D \times D}$ & Query projection matrix at layer $l$ \\
    $\mathbf{b}_K, \mathbf{b}_V$ & $\mathbb{R}^D$ & Bias terms for Key and Value projections \\
    \bottomrule
    \end{tabular}
    }
\end{table}
%%%%% ==================================== VTAB-1k Result ============================================
\begin{table*}
\centering
\scriptsize
\setlength{\tabcolsep}{3pt}
\resizebox{\textwidth}{!}{
\begin{tabular}{l|ccccccc|cccc|cccccccc}
\toprule
~ & \multicolumn{7}{c}{\textit{Natural} (7)} & \multicolumn{4}{|c}{\textit{Specialized} (4)} & \multicolumn{8}{|c}{\textit{Structured} (8)}\\
        Methods & \rotatebox{90}{CIFAR-100} & \rotatebox{90}{Caltech101} & \rotatebox{90}{DTD} 
        & \rotatebox{90}{Flowers102} & \rotatebox{90}{Pets} & \rotatebox{90}{SVHN} & \rotatebox{90}{Sun397} 
        & \rotatebox{90}{Patch Camelyon} & \rotatebox{90}{EuroSAT} & \rotatebox{90}{Resisc45} 
        & \rotatebox{90}{Retinopathy} & \rotatebox{90}{Clevr/count} & \rotatebox{90}{Clevr/distance} 
        & \rotatebox{90}{DMLab} & \rotatebox{90}{KITTI/distance} & \rotatebox{90}{dSprites/loc} 
        & \rotatebox{90}{dSprites/ori} & \rotatebox{90}{SmallNORB/azi} 
        & \rotatebox{90}{SmallNORB/ele} \\
\midrule \hline
\multicolumn{20}{c}{ \textit{ViT-B with supervised pre-trained on ImageNet-21K} } \\ \hline
Full fine-tuning \cite{jia2022visual} & 68.9 & 87.7 & 64.3 & 97.2 & 86.9 & 87.4 & 38.8 & 79.7 & 93.7 & 84.2 & 73.9 & 56.3 & 58.6 & 41.7 & 65.5 & 57.5 & 46.7 & 25.7 & 29.1 \\
Linear probing \cite{jia2022visual} & 63.4 & 85.0 & 63.2 & 97.0 & 86.3 & 36.6 & 51.0 & 78.5 & 87.5 & 68.6 & 74.0 & 34.3 & 30.6 & 33.2 & 55.4 & 12.5 & 20.0 & 9.6 & 19.2  \\
Adapter \cite{adapter19} & 74.1 & 86.1 & 63.2 & 97.7 & 87.0 & 34.6 & 50.8 & 76.3 & 88.0 & 73.1 & 70.5 & 45.7 & 37.4 & 31.2 & 53.2 & 30.3 & 25.4 & 13.8 & 22.1  \\
Bias \cite{zaken2021bitfit} & 72.8 & 87.0 & 59.2 & 97.5 & 85.3 & 59.9 & 51.4 & 78.7 & 91.6 & 72.9 & 69.8 & 61.5 & 55.6 & 32.4 & 55.9 & 66.6 & 40.0 & 15.7 & 25.1  \\

VPT-Deep \cite{jia2022visual}  & \textbf{78.8} & 90.8 & 65.8 & 98.0 & 88.3 & 78.1 & 49.6 & 81.8 & 96.1 & 83.4 & 68.4 & 68.5 & 60.0 & 46.5 & \textbf{72.8} & 73.6 & 47.9 & 32.9 & 37.8  \\
DA-VPT+ (ours) & 74.4 & \textbf{92.7} & \textbf{74.3} & \textbf{99.4} & \textbf{91.3} & \textbf{91.5} & \textbf{50.3} & \textbf{86.2} & \textbf{96.2} & \textbf{87.2} & \textbf{76.3} & \textbf{81.3} & \textbf{62.5} & \textbf{52.8} & 65.3 & \textbf{84.9} & \textbf{51.0} & \textbf{33.1} & \textbf{48.7} \\
\hline
\multicolumn{20}{c}{ \textit{ViT-B with MAE pre-trained on ImageNet-1K} } \\ \hline
Full fine-tuning \cite{jia2022visual} & 24.6 & 84.2 & 56.9 & 72.7 & 74.4 & 86.6 & 15.8 & 81.8 & 94.0 & 72.3 & 70.6 & 67.0 & 59.8 & 45.2 & 75.3 & 72.5 & 47.5 & 30.2 & 33.0 \\
DA-VPT+ (ours) & 38.7 & 87.6 & 64.6 & 83.5 & 86.1 & 83.6 &  22 & 85 & 94.6 & 79.0  & 73.2 & 77.6 & 63.8 & 46.9 & 65.7 & 90.8 & 53.0 & 28.8 & 47.7 \\
\hline
\multicolumn{20}{c}{ \textit{ViT-B with MoCo-V3 pre-trained on ImageNet-1K} } \\ \hline
Full fine-tuning \cite{jia2022visual} & 57.6 & 91.0 & 64.6 & 91.6 & 79.9 & 89.8 & 29.1 & 85.1 & 96.4 & 83.1 & 74.2 & 55.2 & 56.9 & 44.6 & 77.9 & 63.8 & 49.0 & 31.5 & 36.9 \\
DA-VPT+ (ours) & 63.5 & 90.7 & 69.8 & 92.5 & 90.6 & 90.5 & 40.5 & 85.8 & 96.0 & 83.9 & 73.2 & 80.5 & 62.3 & 49.8 & 63.7 & 84.2 & 52.2 & 30.3 & 48.9 \\

\bottomrule
\end{tabular}
}
\caption{Results of details of performance comparisons on the VTAB-1k benchmark with ViT-B/16 models with supervised, MAE and MoCo-V3 pre-training.}
\label{tab:result_vtab_vit}
\end{table*}

\section{More Examples of Attention Maps on Prompts}
We also demonstrate some representative attention visualizations from CUB-200-2011 and Stanford Dogs datasets. For each image, we display the attention map corresponding to its class-assigned prompt. The attention patterns demonstrate how our method learns to focus on class-discriminative regions. 

\begin{figure*}[htbp]
    \centering
    \begin{subfigure}[b]{0.46\textwidth} 
        \centering
        \includegraphics[width=\textwidth]{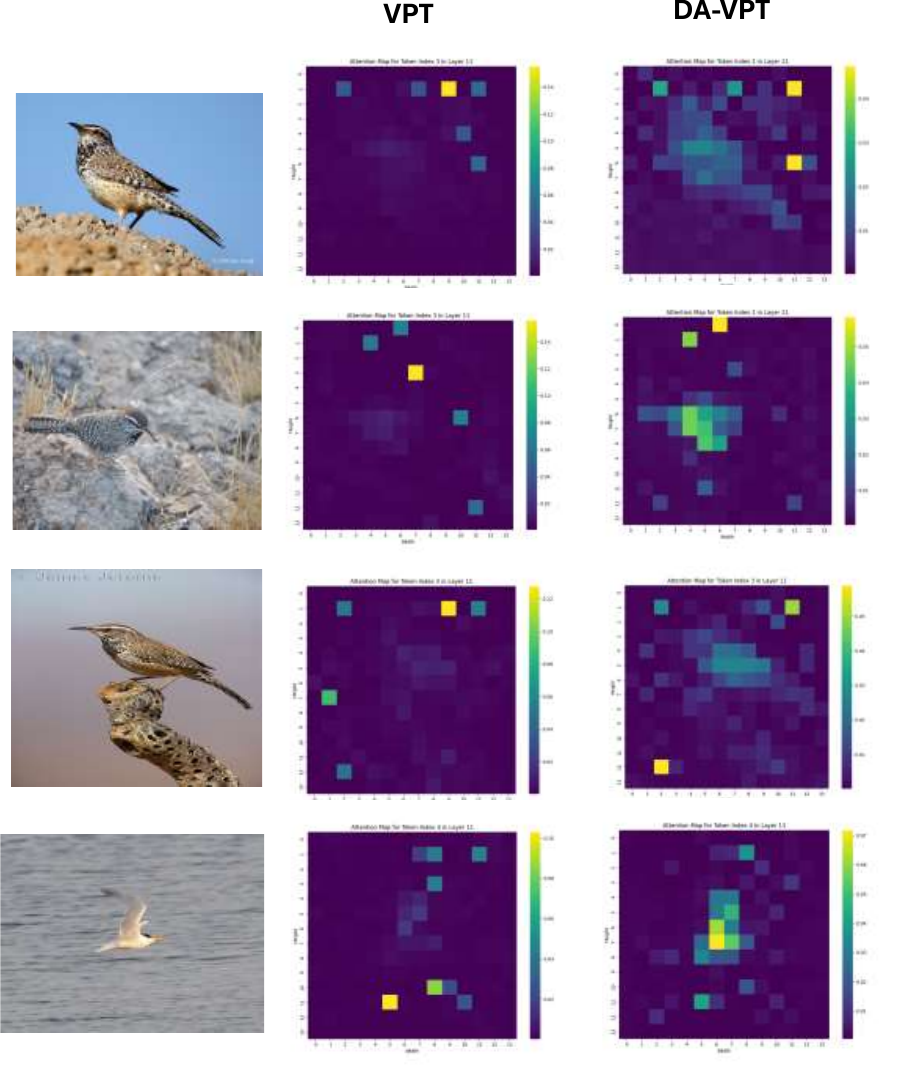}
        \caption{Part 1}
        \label{fig:more_examples:1}
    \end{subfigure}
    \hspace{-0.5em}
    \begin{subfigure}[b]{0.48\textwidth} 
        \centering
        \includegraphics[width=\textwidth]{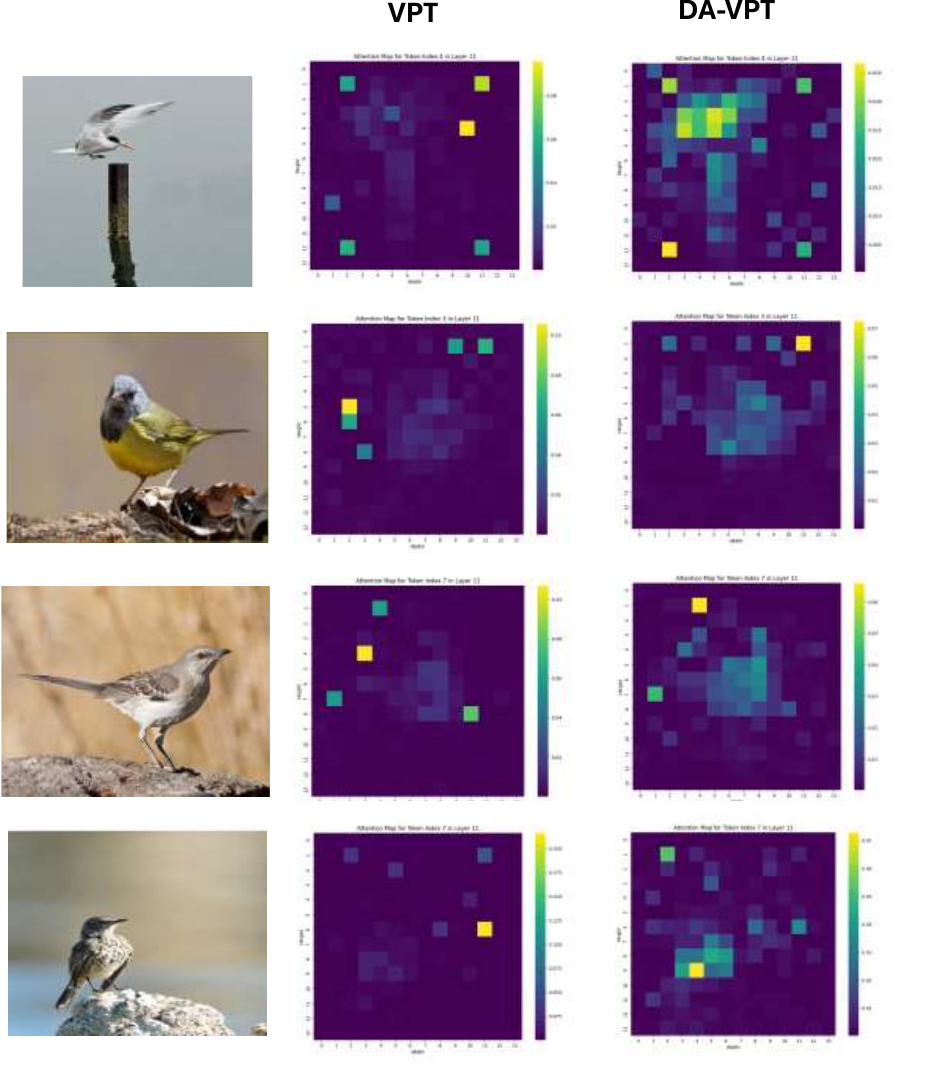}
        \caption{Part 2}
        \label{fig:more_examples:2}
    \end{subfigure}
    \begin{subfigure}[b]{0.50\textwidth} 
        \centering
        \includegraphics[width=\textwidth]{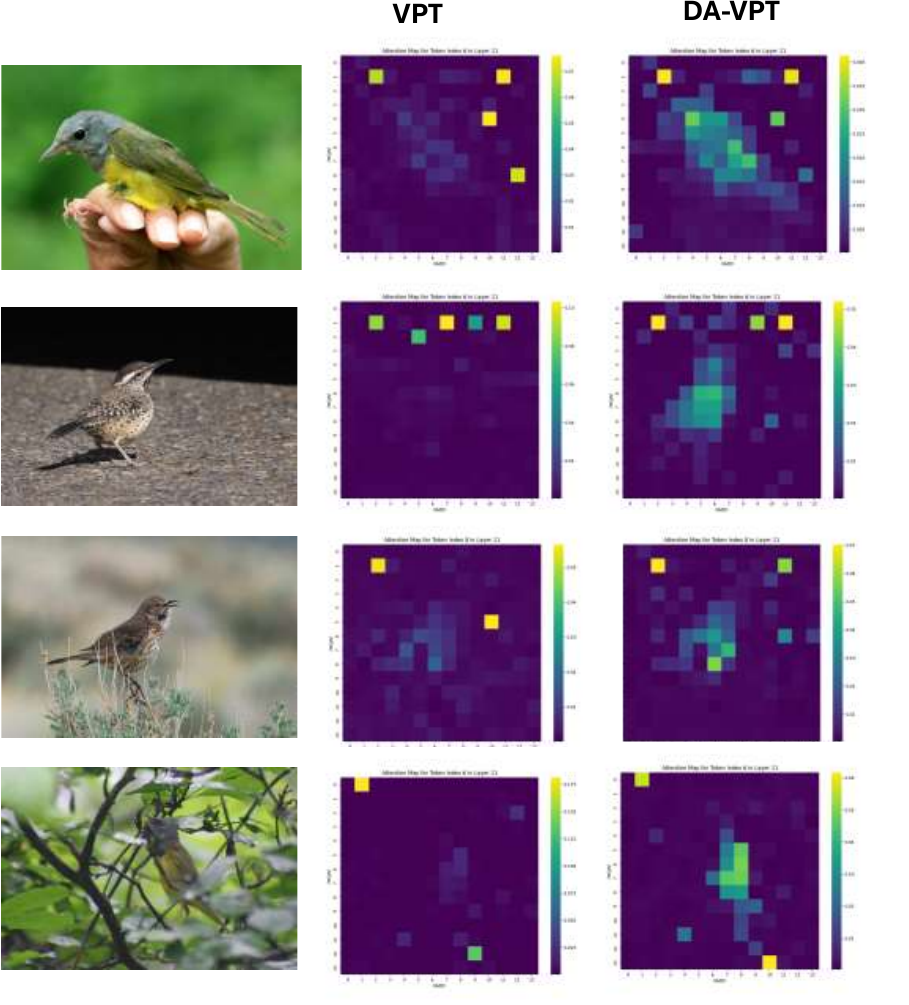}
        \caption{Part 3}
        \label{fig:more_examples:3}
    \end{subfigure}
    \begin{subfigure}[b]{0.46\textwidth} 
        \centering
        \includegraphics[width=\textwidth]{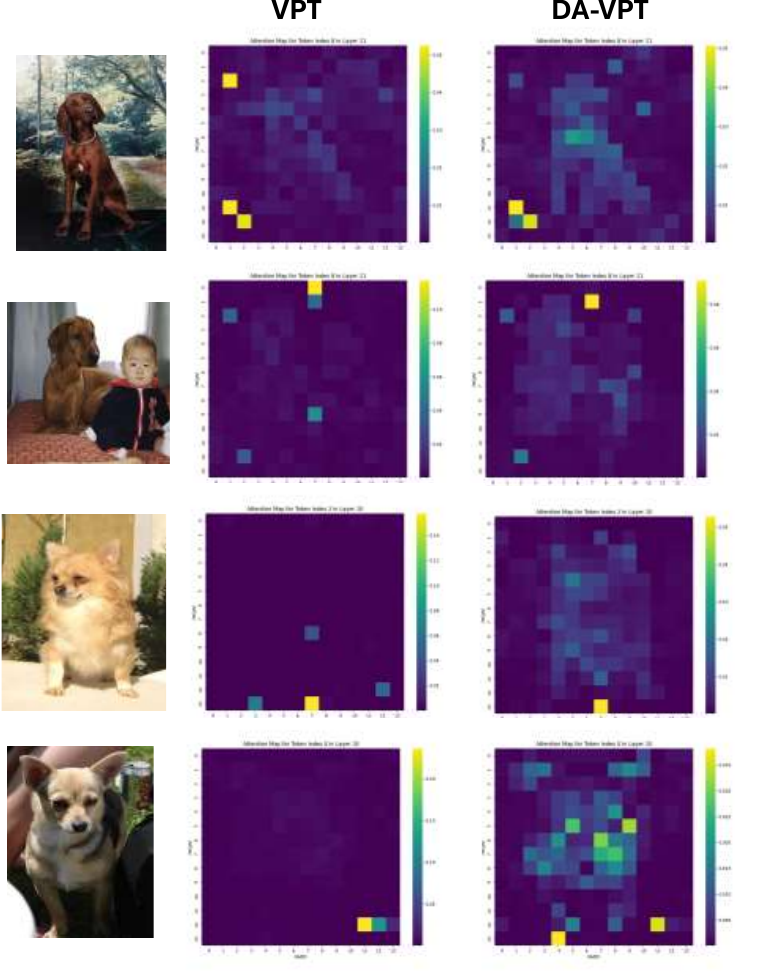}
        \caption{Part 4}
        \label{fig:more_examples:4}
    \end{subfigure}
    \caption{\textbf{Visualization of Class-Specific Attention Maps.} \ref{fig:more_examples:1},\ref{fig:more_examples:2},\ref{fig:more_examples:3}: Examples from CUB-200-2011 showing fine-grained bird features. \ref{fig:more_examples:4}: Examples from Stanford Dogs highlighting breed-specific characteristics. These visualizations illustrate the model's ability to capture class-relevant visual features across different fine-grained classification tasks.}
    \label{fig:more_examples}
\end{figure*}

\end{document}